\documentclass{article} 
\usepackage{iclr2023_conference,times}

\usepackage{amsmath,amsfonts,bm}

\def\eqref#1{equation~\ref{#1}}

\def\1{\bm{1}}

\DeclareMathAlphabet{\mathsfit}{\encodingdefault}{\sfdefault}{m}{sl}
\SetMathAlphabet{\mathsfit}{bold}{\encodingdefault}{\sfdefault}{bx}{n}

\DeclareMathOperator*{\argmin}{arg\,min}

\usepackage{wrapfig}
\usepackage{hyperref}
\usepackage{url}
\usepackage{fancyhdr}
\usepackage{graphicx}
\usepackage{booktabs}
\usepackage{mathtools}
\usepackage{romannum}
\usepackage{xcolor}
\usepackage{amsmath, amssymb, amsfonts, amsbsy, amsthm, latexsym, color}
\usepackage{ulem}
\usepackage{mathtools}
\usepackage{mathrsfs}
\usepackage{cancel}

\usepackage{hyperref} 
\usepackage{empheq}

\usepackage[T1]{fontenc}
\usepackage{authblk}
\usepackage{graphicx}
\usepackage{enumerate}
\usepackage{caption}
\usepackage{subcaption}
\usepackage{caption}
\usepackage{subcaption}
\usepackage{mwe}

\usepackage[ruled,vlined]{algorithm2e}
\newtheorem{theorem}{Theorem}[section]

\theoremstyle{definition}
\newtheorem{definition}{Definition}[section]

\newtheorem{remark}{Remark}[section]

\newcommand{\W}{\widetilde}
\newcommand{\etab}{{\mathbf{\eta}}}

\newcommand{\I}{\mathbf{I}}

\newcommand{\x}{\mathbf{x}}

\usepackage{romannum}
\title{Implicit regularization in Heavy-ball momentum accelerated stochastic gradient descent}

\iclrfinalcopy 

\begin{document}

\author[1]{Avrajit Ghosh$^*$}
\author[1]{He Lyu$^*$}
\author[1]{Xitong Zhang}
\author[1]{Rongrong Wang}
\affil[1]{Computational Mathematics Science and Engineering \linebreak
Michigan State University}

\def\thefootnote{*}\footnotetext{Denotes equal contribution}

\maketitle

\begin{abstract}
It is well known that the finite step-size ($h$) in Gradient Descent (GD) implicitly regularizes solutions to flatter minima. A natural question to ask is \textit{"Does the momentum parameter $\beta$ play a role in implicit regularization in Heavy-ball (H.B) momentum accelerated gradient descent (GD+M)?"} To answer this question, first, we show that  the discrete H.B momentum update (GD+M) follows a continuous trajectory induced by a modified loss, which consists of an original loss and an implicit regularizer. Then, we show that this implicit regularizer for (GD+M) is stronger than that of (GD) by factor of $(\frac{1+\beta}{1-\beta})$, thus explaining why (GD+M) shows better generalization performance and higher test accuracy than (GD). Furthermore, we extend our analysis to the stochastic version of gradient descent with momentum (SGD+M) and characterize the continuous trajectory of the update of (SGD+M) in a pointwise sense. We explore the implicit regularization in (SGD+M) and (GD+M) through a series of experiments validating our theory. 
\end{abstract}
\vspace{-0.3in}
\section{Introduction}
\vspace{-0.1in}
Deep neural networks (NN) have led to huge empirical successes in recent years across a wide variety of tasks, ranging from computer vision, natural language processing, autonomous driving to medical imaging, astronomy and physics \citep{Bengio+chapter2007,Hinton06,goodfellow2016deep}. Most deep learning problems are in essence solving an over-parameterized, large-scale non-convex optimization problem. A mysterious phenomenon about NN that attracted much attention in the past few years is why NN generalizes so well. Indeed, even with extremely overparametrized model, NNs rarely show a sign of over-fitting \citep{neyshabur2017implicit}.
Thus far, studies along this line have successfully revealed many forms of implicit regularization that potentially lead to good generalization when gradient descent (GD) or stochastic gradient descent (SGD) algorithms are used for training, including norm penalty \citep{soudry2018implicit}, implicit gradient regularization \citep{barrett2020implicit}, and implicit Hessian regularization \citep{orvieto2022anticorrelated,orvieto2022explicit} through noise injection. 

 In contrast, the family of momentum accelerated gradient descent methods including Polyak's Heavy-ball momentum \citep{polyak1964some}, Nesterov's momentum \citep{sutskever2013importance}, RMSProp  \citep{tieleman2012lecture}, and Adam \citep{kingma2014adam},  albeit being  powerful alternatives to SGD with faster convergence rates, are far from well-understood in the aspect of implicit regularization.   In this paper, we analyze the implicit gradient regularization in the Heavy-ball momentum accelerated SGD (SGD+M) algorithm with the goal of gaining more theoretical insights on how momentum affects the generalization performance of SGD, and why it tends to introduce a variance reduction  effect whose strength increases with the momentum parameter. 

\vspace{-0.15in}
 \section{Related literature}
 \vspace{-0.1in}

It has been well studied that gradient based optimization  implicitly biases solutions towards models of lower complexity which encourages better generalization. For example, in an over-parameterized  quadratic model, gradient descent with a near-zero initialization implicitly biases solutions towards having a small nuclear norm \citep{arora2019implicit,gunasekar2017implicit,razin2020implicit}, in a least-squares regression problem, gradient descent solutions with 0 initial guess are biased towards having a minimum $\ell_{2}$ norm \citep{soudry2018implicit,neyshabur2014search,ji2019implicit,poggio2020theoretical}. Similarly, in a linear classification problem with separable data, the solution of gradient descent is biased towards the max-margin (i.e., the minimum $\ell_{2}$ norm) solution  \citep{soudry2018implicit}. However in \citep{vardi2021implicit}, the authors showed that these norm-based regularization results proved on simple settings might not extend to non-linear neural networks. 

The first general implicit regularization for GD discovered for all non-linear models (including neural networks) is the Implicit Gradient Regularization (IGR) \citep{barrett2020implicit}. It is shown that the learning rate in gradient descent (GD)  penalizes the second moment of the loss gradients, hence encouraging discovery of flatter optima. Flatter optima usually give higher test-accuracy and are more robust to parameter perturbations \citep{barrett2020implicit}.

  Implicit Gradient Regularization was also discovered for Stochastic Gradient Descent (SGD) \citep{smith2021origin} \cite{JMLR:v20:17-526} , as one (but perhaps not the only one) reason for its good generalization.
   SGD is believed to also benefit from its stochasticity, which might act as a type of noise injection to enhance the performance. Indeed, it is shown in \citep{wu2020noisy} that, by injecting noise to the gradients, full-batch gradient descent will be able to match the performance of SGD with small batch sizes.  Besides injecting noise to the gradients, many other ways of noise injections have been discovered to have an implicit regularization effect on the model parameters, including noise injection to the model space \citep{orvieto2022explicit} and those to the network activations  \citep{camuto2020explicit}. However, how these different types of regularization cooperatively affect generalization is still quite unclear.  

The effect of generalization in momentum accelerated gradient descent has been studied much less. \cite{JMLR:v20:17-526} analyzed the trajectory of SGD+M and found that it can be weakly approximated by solutions of certain Ito stochastic differential equations, which hinted the existence of IGR in (SGD+M). However, both the explicit formula of IGR and its relation to generalization remain unknown.
Recently, in  \citep{https://doi.org/10.48550/arxiv.2110.03891}, the authors analyzed the implicit regularization in momentum (GD+M) based on a linear classification problem with separable data and show that (GD+M) converges to the $L_2$ max-margin solution. Although this is one of the first proposed forms of implicit regularization for momentum based methods, it fails to provide an insight on the implicit regularization for momentum in non-linear neural networks.

Recently, \citep{jelassi2022towards} has shown that the (GD+M) increases the generalization capacity of networks in some special settings (i.e., a simple binary classification problem with a two layer network and part of the input features are much weaker than the rest), but it is unclear to which extent the insight obtained from this special setting can be extended to practical NN models.

To the best of our knowledge, no prior work has derived an implicit regularization for (SGD+M) for general non-linear neural networks. 

\vspace{-0.1in}
\section{Implicit gradient regularization for gradient descent and its relation to generalization}\label{Sec:3}
We briefly review the IGR  defined for GD \citep{barrett2020implicit} which our analysis will be based on. Let $E(\x)$ be the loss function defined over the parameters space $\x \in \mathbb{R}^p$ of the neural network. Gradient descent iterates take a discrete step ($h$) opposite to the gradient of the loss at the current iterate 
  \begin{align}\label{eq:GD}
      \x^{k+1} =  \x^{k} - h \nabla E(\x^{k}).
  \end{align} 

  With an infinitesimal step-size ($h\rightarrow 0$), the trajectory of GD converges to that of the first order ODE 
  \begin{align}\label{eq:GF}
      \x'(t) = -\nabla E(\x(t)) 
  \end{align} 
   known as the gradient flow. But for a finite (albeit small) step size $h$, the updates of GD steps off the path of gradient flow and follow more closely the path of a modified flow: 
    \begin{align}
  \label{IGR}
     \x'(t) = -\nabla \hat{E}(\x(t)), \quad\quad \textrm{where  }  \hat{E}(\x) = E(\x) +\frac{h}{4}\| \nabla E (\x)\|^2.
  \end{align}

 \vspace{-0.1in}
  It is shown \citep{barrett2020implicit} via the so-called classical backward analysis that when GD and the two gradient flows \ref{eq:GF} and \ref{IGR} all set off from the same point $\x^k$, the next gradient update $\x^{k+1}$ is $O(h^2)$ close to the original gradient flow (\ref{eq:GF}) evaluated at the next time point $t_{k+1} = h + t_k$, but is $O(h^3)$ close to the modified flow (\ref{IGR}) evaluated at $t_{k+1}$. So, locally, the modified flow  tracks the gradient descent trajectory more closely than the original flow. To discuss the global behaviour of GD, we need the following definition of closeness between two trajectories.
 
 \begin{definition}[$O(h^{\alpha})$-closeness in the strong sense]\label{def:sp} Fix some $T>0$, we call the trajectory of the discrete GD-update $\x^k$ and a continuous flow $\tilde{\x}(t_{k})$ to be $O(h^{\alpha})$-close in the strong sense \footnote{Weak-sense approximation as studied in \citet{JMLR:v20:17-526} only requires the distributions of the sample processes $\x$ and $\tilde{\x}$ to be close, whereas our strong-sense approximation requires each instance of $\x$ and $\tilde{\x}$ to be close. The strong-sense IGR found using the latter is valid for trainings with any fixed random-batch sequence and any fixed initialization, while the former only characterizes the mean trajectory taking expectation over many different trainings (each with a random batch-sequence and initialization).} if:
 \begin{align*}
  \label{alpha-close}  
     \max_{k \in \mathcal{K}} \|\x^k- \tilde{\x}(t_{k}) \|_{2} \leq c h^{\alpha}, \quad  \textrm{ $\mathcal{K} =
     \left\{ 1,... \left\lfloor \frac{T}{h} \right\rfloor\right\}$}, 
  \end{align*}
 where $t_k = hk$, and $c$ is some constant independent of $h$ and $k$
 \end{definition} 
 
 \vspace{-0.1in}
 
   Definition \ref{def:sp} quantifies the global closeness of a discrete trajectory and a continuous one. By accumulating the local error, one can show that setting off from the same location $\x^0$, the original gradient flow \eqref{eq:GF} is $O(h)$-close to the GD trajectory while that of the modified flow \ref{IGR} is $O(h^2)$-close. Based on this observation, the authors defined the \textcolor{red}{$O(h)$} term $\frac{h}{4}\| \nabla E (\x)\|^2 $ in the modified flow as the IGR term and concluded by stating that it guides the solutions to flatter minima in a highly non-convex landscape.
  
 To justify why minimizing $\| \nabla E (\x)\|^2$ is a good idea and why it encourages flat minimizers (which seems to be missing from the original paper), we borrow an argument from  \citep{foret2020sharpness} that was originally developed for a different purpose.

  Due to the PAC-Bayes analysis \citep{neyshabur2017implicit}, a simplified generalization bound for NN derived under some technical conditions can be stated as \citep{foret2020sharpness}
   \[
 \mathcal{L}_{\mathcal{D}} (\x) \leq \max\limits_{\|\epsilon\|\leq \rho}
\mathcal{L}_{\mathcal{S}}(\x + \epsilon) + \hat{h}\left(\frac{\| \x\|^2}{\rho}\right), \textrm{  for any  } \rho>0,
 \]
  where  $\mathcal{L}_{\mathcal{D}} $ is the population loss (i.e., the generalization error),  $\mathcal{L}_{\mathcal{S}}$ is the empirical/training loss, and $\hat{h}: \mathbb{R}_+ \rightarrow \mathbb{R}_+$ is some strictly increasing function. One can try to minimize this upper bound in order to minimize the generalization error $\mathcal{L}_{\mathcal{D}}$. The term $\hat{h}\left(\frac{\| \x\|^2}{\rho}\right)$ in the upper bound can be controlled by activating a weight decay penalty during training, and the first term in the bound is usually written into
  \[
  \max\limits_{\|\epsilon\|\leq \rho} \mathcal{L}_{\mathcal{S}}(\x + \epsilon) =  \underbrace{\max\limits_{\|\epsilon\|\leq \rho} (\mathcal{L}_{\mathcal{S}}(\x + \epsilon) -  \mathcal{L}_{\mathcal{S}}(\x))}_{\textrm{sharpness}} + \underbrace{\mathcal{L}_{\mathcal{S}}(\x)}_{\textrm{training loss}}
  \]
  which consists of the training loss and an extra term called sharpness, minimizing which will  help with generalization.
  Since directly minimizing the sharpness is difficult, one can then use the following first-order Taylor approximation
  \[
 \max\limits_{\|\epsilon\|\leq \rho}(\mathcal{L}_{\mathcal{S}}(\x + \epsilon) -  \mathcal{L}_{\mathcal{S}}(\x)) \approx \max\limits_{\|\epsilon\|\leq \rho}  \epsilon^T\nabla \mathcal{L}_{\mathcal{S}}(\x) = \rho \|\nabla \mathcal{L}_{\mathcal{S}}(\x)\|.
  \]
  Using our notation, $\nabla \mathcal{L}_{\mathcal{S}}(\x)$ is $\nabla E(\x)$, so the sharpness is approximately proportional to $\|\nabla E\|$ which is the square root of the IGR term. 


   \section{Implicit gradient regularization for Heavy ball accelerated gradient descent }
 
 The main mathematical challenge in studying the IGR in momentum updates is that we now need to perform global error analysis instead of the local backward analysis, as the momentum updates utilizes the entire update history.

 The IGR for Heavy-ball momentum was previously analyzed in \citep{kovachki2021continuous} through studying its relationship with 
 the damped second order Hamiltonian dynamic
\[
m \x'(t) + \gamma \x''(t) + \nabla E (\x(t)) =0
\]
which has been well-known as the underlying ODE for the momentum updates. However, only $O(h)$ closeness is proven between the momentum updates and this ODE trajectory under general step size assumptions \footnote{An $O(h^2)$ closeness is proven under very stringent conditions on the learning rate, which excludes the interesting regime where momentum has an advantage over plain GD in terms of the convergence rate and stability}, which is not enough since the implicit regularization term $\frac{h}{4}\| \nabla E\|^2$ itself is of order $O(h)$. In addition, this approach is difficult to be applied to the stochastic setting. 

In this paper, we circumvent the use of the second order ODE (as it only gives $O(h)$ closeness) and directly obtain a continuous path that is $O(h^2)$-close to the momentum update for both GD and SGD. This is achieved by  linking the momentum updates with a first order piecewise ODE, proving that the ODE  has a piece-wise differentiable trajectory that is $O(h^2)$-close to the momentum updates, and then using its trajectory to study the IGR. 
The detailed argument can be found in the appendix. Here we provide the final mathematical formula for the implicit regularization of (Heavy-Ball) momentum based gradient descent method (IGR-M).
\begin{theorem}
(IGR-M): Let the loss for the full-batch gradient $E$ be smooth and 4-times differentiable, then
the (GD+M) updates
\begin{align*}
    \x^{k+1} = \x^{k} - h \nabla E(\x^{k}) + \beta(\x^{k} -\x^{k-1})  \quad \forall{k=1,2,...,n}
\end{align*} 
are $O(h^2)$ close to 
the flow of the continuous trajectory of the piecewise first-order ODE 
\begin{align}\label{eq:ode}
      \W\x^{'}(t) =  - \frac{1-\beta^{k+1}}{1-\beta} \nabla E(\W\x(t)) - \frac{h\gamma (1+\beta)}{2(1-\beta)^3} \nabla^{2} E(\W\x(t)) \nabla E(\W\x(t)), t \in [t_k, t_{k+1} ]
\end{align}
where $t_k = kh$ and
\[\gamma = (1-\beta^{2k+2})-4(k+1)\beta^{k+1}\frac{(1-\beta)}{(1+\beta)}. \]
Since $\beta^{k}$ quickly decays to 0 as $k$ grows, for a sufficiently large iteration $k$,   \eqref{eq:ode} reduces to:
 \begin{align}
 \label{eq:IGRM}
    \W \x^{'}(t) = -\frac{1}{1-\beta} \nabla \hat{E}(\W\x(t)), \quad t\in[0,T]
 \end{align}

driven by the modified loss 
$\hat{E}(\W\x(t)):=  E(\W\x(t)) + \frac{(1+\beta)h}{4(1-\beta)^2} \|\nabla  E(\W\x(t))\|_{2}^2$. More specifically, for a fixed time $T$, there exists a constant $C$, such that for any learning rate $0<h<T$, we have
\begin{equation}
    \|\W\x(t_{k})- \x^{k}\|_{2} \leq Ch^2,\  t_k = kh, \quad  { k =1,2,..., \lfloor\frac{T}{h}\rfloor}.
\end{equation}
\end{theorem}
\vspace{-0.15in}
Comparing the continuous trajectory traced with \ref{eq:IGRM} and the one without momentum \ref{IGR}, we immediately have a few important observations:
\begin{remark}\label{remark-gd-1} Ignoring the $O(h)$ terms in \ref{eq:IGRM} and \ref{IGR}, we see that the momentum trajectory is driven by a force that is $\frac{1}{1-\beta}$ times as large as that for GD. Therefore, using the same learning rate,  (GD+M) is expected to converge $\frac{1}{1-\beta}$ times as fast as GD. Alternatively, (GD+M) with a learning rate $h$ has roughly the same convergence rate as GD with a learning rate $\frac{h}{1-\beta}$. From now on, we call $\frac{h}{1-\beta}$ the effective learning rate of (GD+M). 
\end{remark}

\begin{remark} \label{remark-gd-2} In terms of the IGR, we can see that adding the momentum amplifies the strength of the IGR (the coefficient in front of the IGR term increased from the $\frac{h}{4}$ in (GD) to the $\frac{h}{4} \frac{1+\beta}{(1-\beta)^2}$ in (GD+M). Even when letting the effective learning rates of the two methods equal (i.e., $\frac{h_{GD+M}}{1-\beta} = h_{GD}$), the implicit regularization in (GD+M) is still $\frac{1+\beta}{1-\beta}$ times stronger. 
\end{remark}
\begin{remark}
The IGR for (GD+M) reduces to the IGR for GD when $\beta=0$.
\end{remark}
\vspace{-0.1in}
With an additional momentum parameter $\beta$, the strength of the implicit regularizer increases by a factor of $\frac{1+\beta}{1-\beta}$. Hence for increasing values of momentum parameter $\beta$, the strength of the regularization increases, thus preferring more flatter trajectories through its descent.


\vspace{-0.1in}
\subsection{IGR-M in a 2D Linear model}
\label{2d}
\vspace{-0.1in}
We first show the impact of IGR-M in a very simple setting that minimizes a loss function with two learnable parameters, i.e., $(\hat{w_{1}},\hat{w_{2}}) = \argmin_{w_{1},w_{2}} E(w_{1},w_{2} ) $ where $E(w_{1},w_{2} ) = \frac{1}{2} (y-w_{1}w_{2}x)^2 $. Here $x,y,w_{1},w_{2}$ are all scalars and mimics a simple scalar linear two-layer network. For a given scalar $(x,y)$, the global minima of $E(w_{1},w_{2})$ are all the points on the curve $w_{1}w_{2} =\frac{y}{x}$. The continuous gradient flow of the parameters are $w_{1}'(t) =- \frac{\partial E }{\partial w_{1}}  $ and $ w_{2}'(t) =- \frac{\partial E }{\partial w_{2}}$.
\begin{wrapfigure}{r}{0.55\textwidth}
  \begin{center}
    \includegraphics[width=0.53\textwidth]{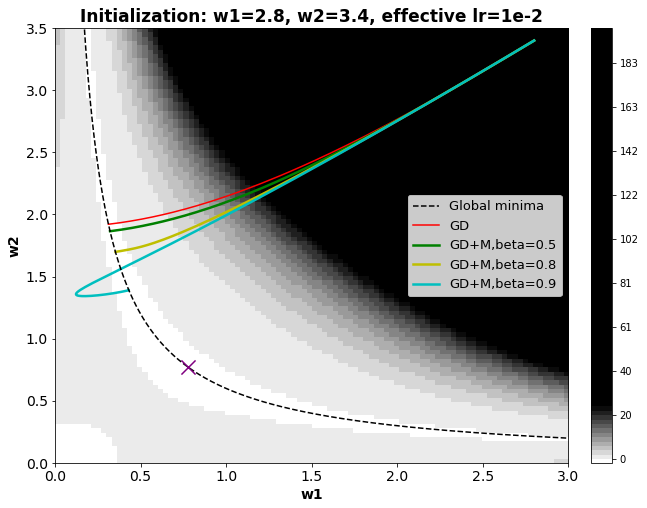}
  \end{center}
  \caption{Implicit regularization for (GD+M) is stronger than that of (GD) for the same effective learning rate $\frac{h}{(1-\beta)}$. As $\beta$ increases the optima seems to  find solution with a lower norm. This confirms Remark \ref{remark-gd-2} that the strength of implicit regularization increases with $\beta$. The background color denotes the magnitude of the norm of the gradient, i.e. , $\|\nabla E \|_{2}^2$}
\end{wrapfigure}
The IGR  for this problem is $\frac{h_{GD}}{4} \| \nabla E\|_{2}^2 = \frac{h_{GD}}{4}(w_{1}^2 + w_{2}^2)E$, which will regularize the trajectory to find minima with a smaller value $(w_{1}^2 + w_{2}^2)$ (towards the cross) among all the global minima lying on $w_{1}w_{2} = \frac{y}{x}$. We intentionally chose the same experiment as in \citep{barrett2020implicit} to compare the effect of implicit regularization for (GD) and (GD+M).

For a fair comparison between (GD) and (GD+M), we set the effective learning rates to be the same, i.e, $h_{GD}=\frac{h_{GD+M}}{1-\beta}$ as in Remark \ref{remark-gd-1}. 
 With the same initialization $(w^{0}_{1},w^{0}_{2}) = (2.8,3.4)$ we explore and track the path of four trajectories with $(h_{GD+M},\beta)$ being $(5\times 10^{-3},0.5)$, $(2\times 10^{-3},0.8) $, $(10^{-3},0.9) $ and $(10^{-2},0.0)$. In all the four cases, the effective learning rates are the same, i.e, $\frac{h_{GD+M}}{(1-\beta)} = h_{GD} = 10^{-2}$. We make the following observations: a) For all the four trajectories, the converged weight parameters $w^{*}=(w_{1}^{*},w_{2}^{*})$ lie on the global minima curve. b) With increasing value of $\beta$, the converged solutions have decreasing value of $\ell_{2} $ norm (or increasing strength of implicit regularization), i.e, $ \|w^{*}_{(10^{-3},0.9)}\|_{2} <  \|w^{*}_{(2\times 10^{-3},0.8)}\|_{2} < \|w^{*}_{(5\times 10^{-3},0.5)}\|_{2} < \|w^{*}_{(10^{-2},0.0)}\|_{2}$.  This observation supports Remark \ref{remark-gd-2}, that the strength of implicit regularization increases with $\beta$, even with the effective learning rate.  
\vspace{-0.15in}
 \section{Implicit regularization in SGD with momentum}
 \vspace{-0.15in}
 In SGD, the full-batch gradient is replaced by it's sampled unbiased estimator. Assume that the loss function $E(\x)$ has the following form
 \begin{equation}
     E(\x) = \frac{1}{M}\sum_{j=1}^M E_{(j)}(\x)
 \end{equation}
 where $E_{(j)}$ is the $j^{th}$ mini-batch loss.
 In the $k^{th}$ iteration, we randomly pick a mini-batch, whose loss is denoted by $E_k$, and update the parameters accordingly.
 The heavy-ball accelerated SGD iterates as 
 \begin{equation}
\label{eq:ball}
\left\{
\begin{aligned}
&\x^{k+1} = \x^{k} - h \nabla E_{k}(\x^{k}) + \beta(\x^{k} -\x^{k-1}) & &  { k =1,2,...,n}\\
&\x^{1} = \x^{0}-  h \nabla E_{0}(\x^{0})\\
&\x^{0} = \x^{-1}=\mathbf{0} \\ 
\end{aligned}
\right.
\end{equation} 

 For each iteration $k$, the update is driven by the current mini-batch loss $E_{k}$. Its continuous approximation is
 \begin{align}
 \label{van_sgd}
     & \x'(t) = -\nabla E_{k}(\x(t)) + \frac{\beta}{h}(\x(t_k)-\x(t_{k-1})) \quad \text{for  $t_{k} < t <t_{k+1} $}
 \end{align}
  during the time $[t_{k},t_{k+1}]$. As a result, the trajectory of $\x$ is continuous but piece-wise differentiable as it is easy to see that the left and right-side derivatives are not equal at a transit point $t_{k}$ from one batch to another,
 $\x'(t_{k}^{+})\neq \x'(t_{k}^{-}) $. Therefore, we expect the implicit regularization term to also have discontinuous derivatives on different intervals. Below we present the mathematical formula for IGR-M in the stochastic setting.

 \begin{theorem}
\label{main_theorem}[IGR-M stochastic version (IGRM$_{s}$)]
Let the loss for each mini-batch $E_{k}$ be 4-times differentiable, then the Heavy Ball momentum updates \ref{eq:ball} are $O(h^2)$ close to the trajectory of the gradient flow with the modified loss,  
\begin{align}
\label{ode_main}
   & \W \x{'} (t) = - \nabla \hat{E}_k(\W\x(t))  \quad  \text{ for $t_{k} \leq t < t_{k+1}$} \notag  \\
   & \text{where,} \quad  \hat{E}_k( \W \x) = \underbrace{G_{k}( \W \x)}_{force} + \underbrace{\frac{h}{4} ( \|\nabla G_{k}( \W \x)\|_2^2 + 2 \sum_{r=0}^{k-1} \beta^{k-r} \|\nabla G_{r}( \W \x)\|_2^2 )}_{IGRM_s}
\end{align}
with $  G_{k}(\W\x(t)) = \sum_{r=0}^k \beta^{k-r}  E_{r}(\W\x(t)) $. Specifically, there exists a constant $C$ such that 
\begin{equation}
    \|\W\x(t_{k})- \x^{k}\|_{2} \leq Ch^2,\   \text{ k =1,2,...,n}. \notag
\end{equation}

\end{theorem}
\vspace{-0.15in}
The theorem tells us that the IGR for momentum is in the form of $\ell_2$ norms of $\nabla G_k$ which is a weighted sum of the historical gradients $\nabla E_k$, $k=0,...,n$, by powers of $\beta$ and evaluated at the current location $\tilde{\x}(t)$. In addition, the stochastic IGR-M has different expressions on different intervals, caused by the abrupt changes between adjacent batches. 
Some further remarks about the (SGD+M) algorithm are listed below.\\
\vspace{-0.1in}
\begin{remark} The flow of the continuous trajectory is governed by the driving-force  $-\nabla G_{k}(\x(t))$ and the negative gradient of an implicit regularizer $ IGRM_{s}(\W \x) = \frac{h}{4} ( \|\nabla G_{k}( \W \x)\|_2^2 + 2 \sum_{r=0}^{k-1} \beta^{k-r} \|\nabla G_{r}( \W \x)\|_2^2 $ which depends on both the learning rate $h$ and momentum $\beta$.
\end{remark}
\begin{remark}
When $\beta=0$, \ref{ode_main} reduces to 
\begin{equation}\label{eq:igr}
\W\x'(t) = -\nabla \left(\underbrace{E_k(\W\x(t)}_{force} +\underbrace{\frac{h}{4} \|\nabla E_k(\W\x(t))\|_2^2)}_{IGR}\right), \quad t \in [t_k,t_{k+1}],
\end{equation}
which is the modified loss for SGD.
\end{remark}
\begin{remark} \label{remark-sgd-2}  Taking expectation over the random selections of batches, we get
$\mathbb{E}(\textrm{IGRM}_{s})(\x)= \frac{h(1+\beta)}{4(1-\beta)^3}\|\nabla E\|^2+\frac{h}{4(1-\beta)^2}F$, where $F:= \mathbb{E} \|E_n-E\|^2$ (appendix Th 3.1). In comparison, the IGR term in SGD after taking expectation is $\frac{h}{4}(\|\nabla E\|^2 +F$) \citep{smith2021origin}, which is much weaker. Even with the adjusted learning rate, the IGR in (SGD+M) is still about $1/(1-\beta)$ times stronger than SGD.
 
\end{remark}

\begin{remark}[Variance reduction]\label{remark-sgd-1} We notice that momentum has a variance reduction effect.  Explicitly, suppose the effective learning rate (Remark \ref{remark-gd-1}) is used so that the force terms in (SGD) and  (SGD+M) have the same expectation, and then we can compare their variance. Let the covariance matrix of $\nabla E_k$ at a fixed point $\x$ be $C: =  \mathbb{E}(\nabla E_k(\x) -\nabla E(\x))(\nabla E_k(\x) -\nabla E(\x))^T$. Here $\nabla E(\x) $ denotes the full-batch gradient. Then the covariance matrix of the force $-E_k$  driving (SGD) is exactly $C$,  while that of the force $-G_k$ driving (SGD+M) is only $\frac{(1-\beta)}{1+\beta}C$ (appendix Th 4.1), which is about $\frac{(1-\beta)}{(1+\beta)}$ times smaller. 
\end{remark}



\begin{remark} It is observed and confirmed by many literature that a larger variance of SGD iterations (caused by either a small batch size or a larger learning rate \citep{smith2017bayesian,li2017stochastic}) increases the generalization power. Larger variance in mini-batch gradients increases the escape efficiency of SGD from bad local minimas \citep{ibayashi2022quasipotential} [See Appendix section-7 for detailed discussion], hence increasing generalization power. Therefore, the variance reduction effect of (SGD+M) hurts generalization. On the other hand,
the fact that (SGD+M) has a stronger IGR (Remark \ref{remark-sgd-2}) and that (SGD+M) is more stable than (SGD) to the use of large effective learning rates (see e.g., \citep{cohen2021gradient}) tend to help with its generalization. This explains why in practice we usually observe that (SGD+M) is only slightly better than (SGD).  

\end{remark}

  \begin{figure*}[t]
    \centering
        \begin{subfigure}[b]{0.375\textwidth}
            \centering
            \includegraphics[width=\textwidth]{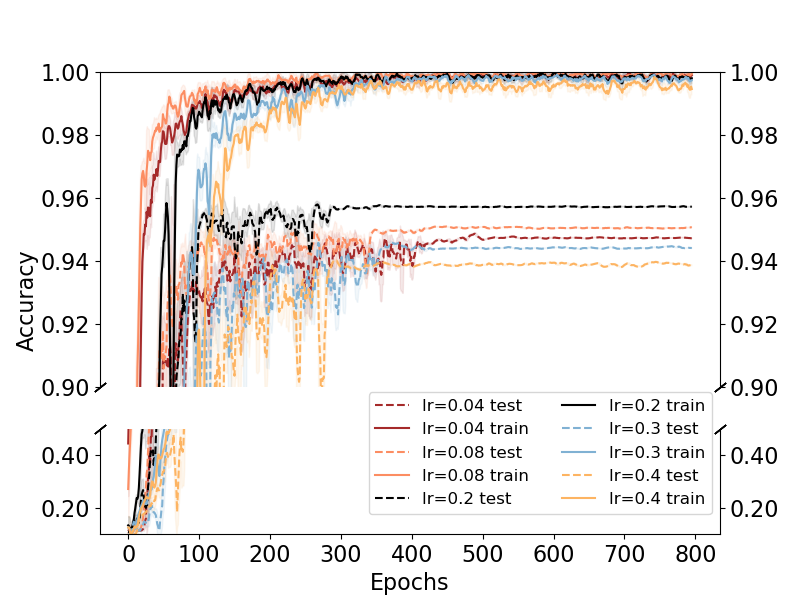}
            \caption[Wideresnet]%
            {{\small Accuracy }}    
            \label{fig:mean and std of net14}
        \end{subfigure}
        \begin{subfigure}[b]{0.375\textwidth}  
            \centering 
            \includegraphics[width=\textwidth]{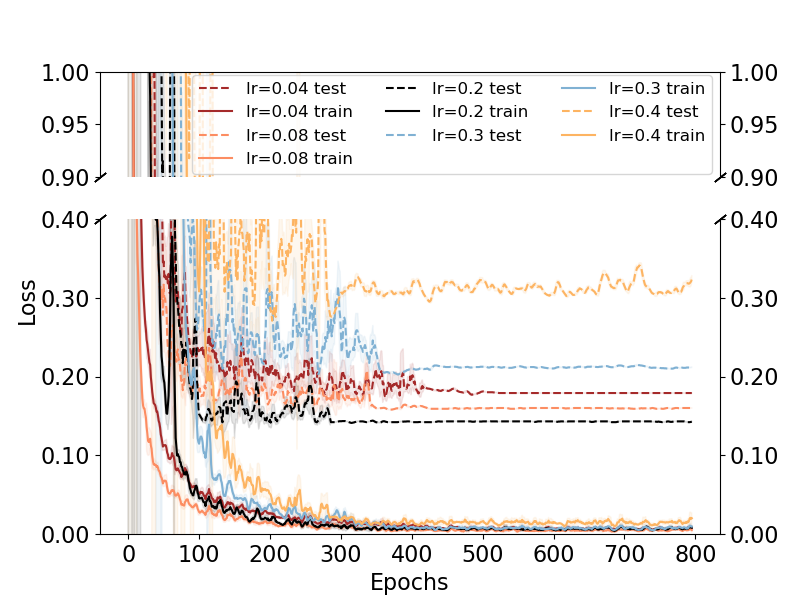}
            \caption[ Resnet-18]%
            {{\small Loss}}    
            \label{fig:mean and std of net24}
        \end{subfigure}
        \caption[ The average and standard deviation of critical parameters ]
        {\small (GD): Classification results of ResNet-18  on MNIST dataset performed with full-batch gradient descent. Figure shows the effect of implicit regularization due to the finite learning rate $h$. Test accuracy improves with higher learning rate till $h=0.2$.  For $h=0.2$, the best test-accuracy is achieved at $ 95.72 \%$  } 
        \label{fig:IGR}
    \end{figure*}
    
     \begin{figure*}[t]
        \centering
        \begin{subfigure}[b]{0.375\textwidth}
            \centering
            \includegraphics[width=\textwidth]{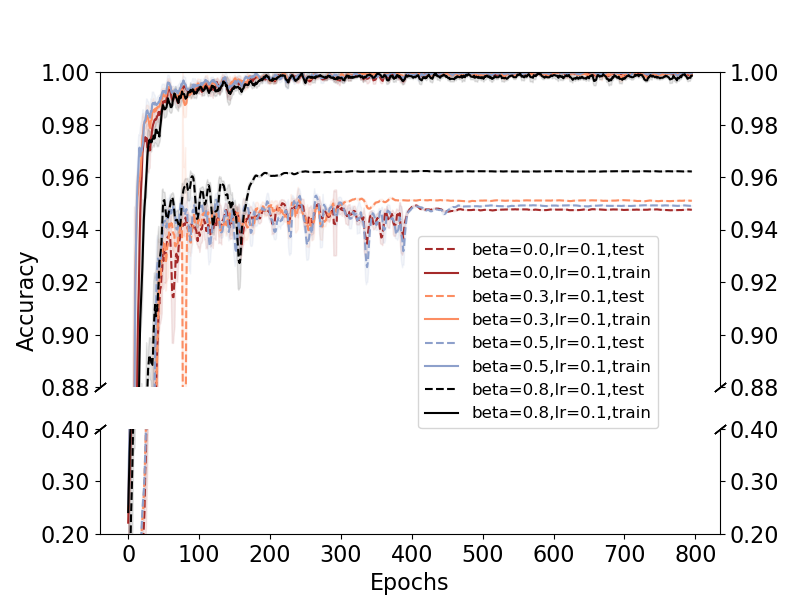}
            \caption[Wideresnet]%
            {{\small Accuracy }}    
            \label{fig:mean and std of net14}
        \end{subfigure}
        \begin{subfigure}[b]{0.375\textwidth}  
            \centering 
            \includegraphics[width=\textwidth]{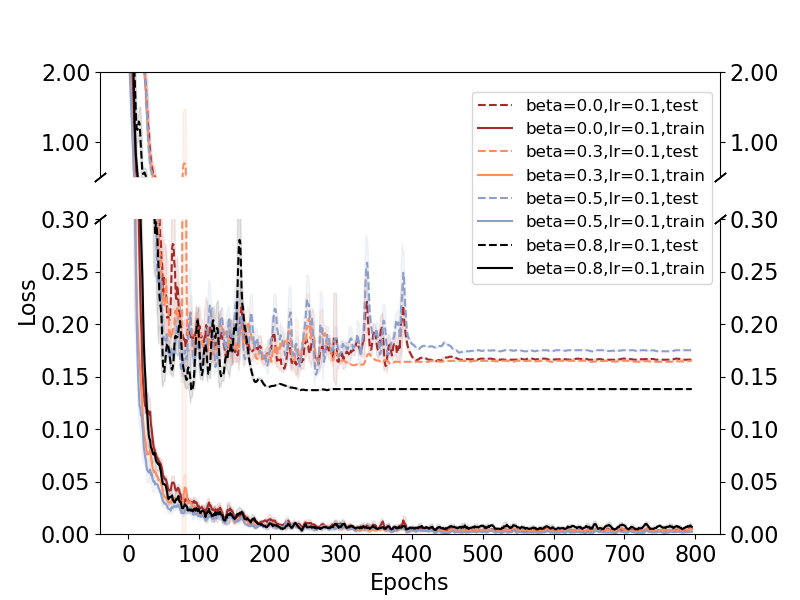}
            \caption[ Resnet-18]%
            {{\small Loss}}    
            \label{fig:mean and std of net24}
        \end{subfigure}
        \caption[ The average and standard deviation of critical parameters ]
        {\small (GD+M): Classification results of ResNet-18 on MNIST dataset trained with various values of momentum parameter $\beta$ for full-batch gradient descent. The best test-accuracy is reported to be $96.22 \% $  } 
        \label{fig:varbeta}
\end{figure*}

\section{Numerical experiments}

Our first experiment is to compare the full-batch (GD) with (GD+M). For a linear least-squares problem with a Hessian matrix bounded by $L$ in the spectral norm, it is well-known that (e.g., \citep{cohen2021gradient}) (GD+M) is stable as long as $h \leq \frac{2+2\beta}{L}$, and GD is stable as long as $h \leq \frac{2}{L}$. This means, the maximum achievable effective learning rate by (GD) is $\frac{2}{L}$, while that by (GD+M) can be as large as $\frac{2(1+\beta)}{L(1-\beta)}$. Since larger effective learning rates means a stronger IGR, (GD+M) clearly benefits from its large stability region. To confirm this, ResNet-18 is used to classify a uniformly sub-sampled MNIST dataset with 1000 training images. All external regularization schemes except learning rate decay and batch normalization have been turned off. We perform the following experiments  : $\mathbf{1}$) Full-batch gradient descent with $\beta=0$ (i.e., GD)   for various learning rate $h$ and the best test accuracy is noted (in Figure \ref{fig:IGR}) to be $95.72\%$.  $\mathbf{2}$)  Full-batch gradient descent with momentum (GD+M) performed for various $\beta$
  with a fixed step-size $h=0.1$ and the best test-accuracy is noted (in Figure \ref{fig:varbeta}) to be $96.22\%$. Our observation is that the best performance of GD (across all learning rates) is worse than the best performance of (GD+M) (across all $\beta$'s). This observation failed to be explained by the known theory of edge of stability\footnote{ 
  ``edge of stability'' (EOS) \cite{cohen2021gradient} is a  phenomenon that shows during network training by the full batch gradient descent, the sharpness $|| \nabla^2 E ||_{2} $ tends to progressively increase until it reaches the threshold $\frac{2}{h}$ and then hovers around it. For GD+M, the sharpness will hover around a large value $\frac{2(1+\beta)}{h}$. Since larger sharpness usually means worse generalization, the EOS theory then predicts that adding momentum hurts the generalization. }
  but can be well-explained by our implicit regularization theory for (GD+M) as adding momentum increases the strength of the IGR.

\begin{figure*}[t]

        \centering
        \begin{subfigure}[b]{0.375\textwidth}
            \centering
            \includegraphics[width=\textwidth]{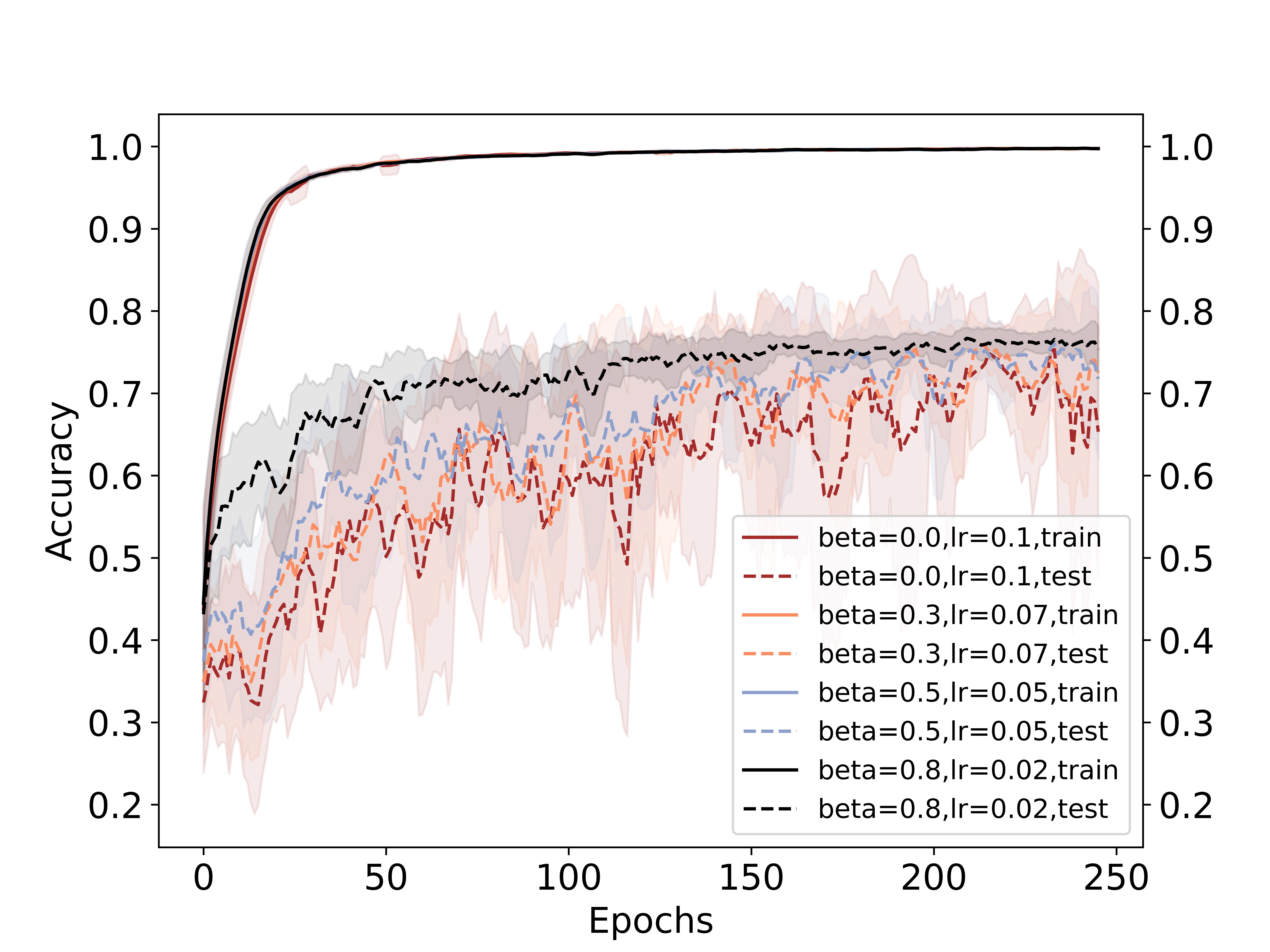}
            \caption[Wideresnet]%
            {{\small WideresNet-16-8}}    
            \label{fig:mean and std of net14}
        \end{subfigure}
        \begin{subfigure}[b]{0.375\textwidth}  
            \centering 
            \includegraphics[width=\textwidth]{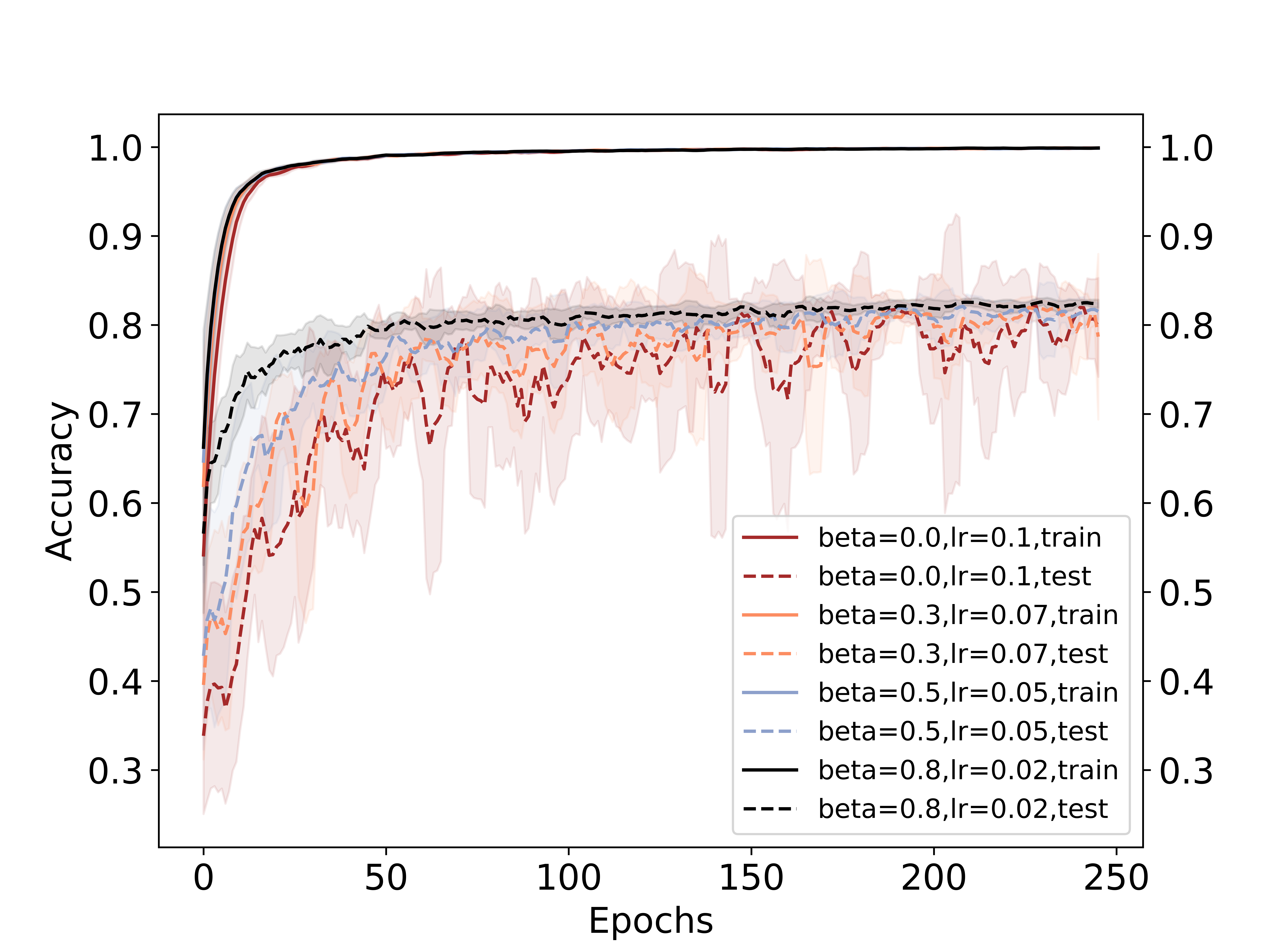}
            \caption[ Resnet-18]%
            {{\small ResNet-18}}    
            \label{fig:mean and std of net24}
        \end{subfigure}
        \vskip\baselineskip \vspace{-0.15 in}
        \begin{subfigure}[b]{0.375\textwidth}   
            \centering 
            \includegraphics[width=\textwidth]{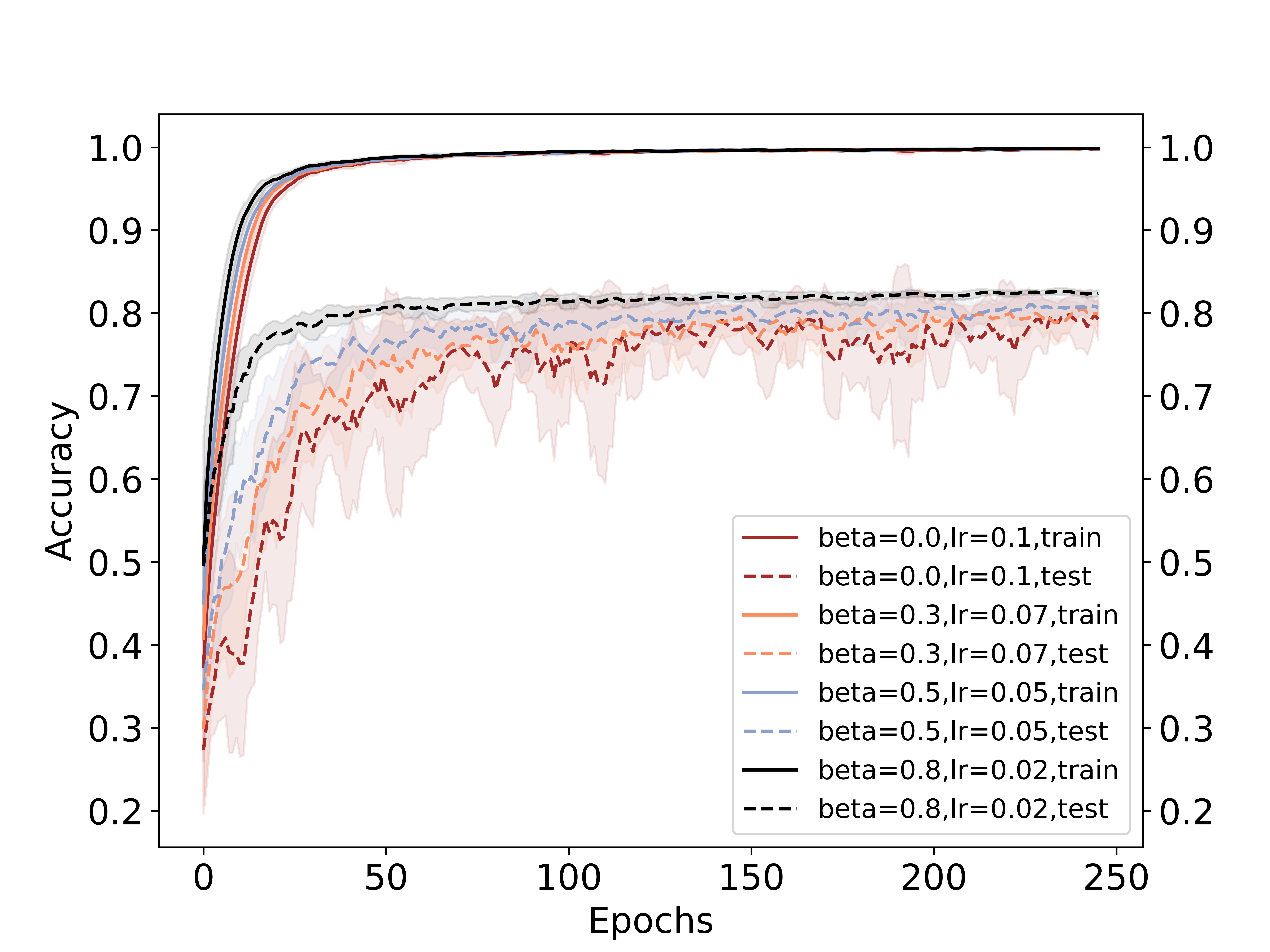}
            \caption[Resnet-50]%
            {{\small ResNet-50}}    
            \label{fig:mean and std of net34}
        \end{subfigure}
        \begin{subfigure}[b]{0.375\textwidth}   
            \centering 
            \includegraphics[width=\textwidth]{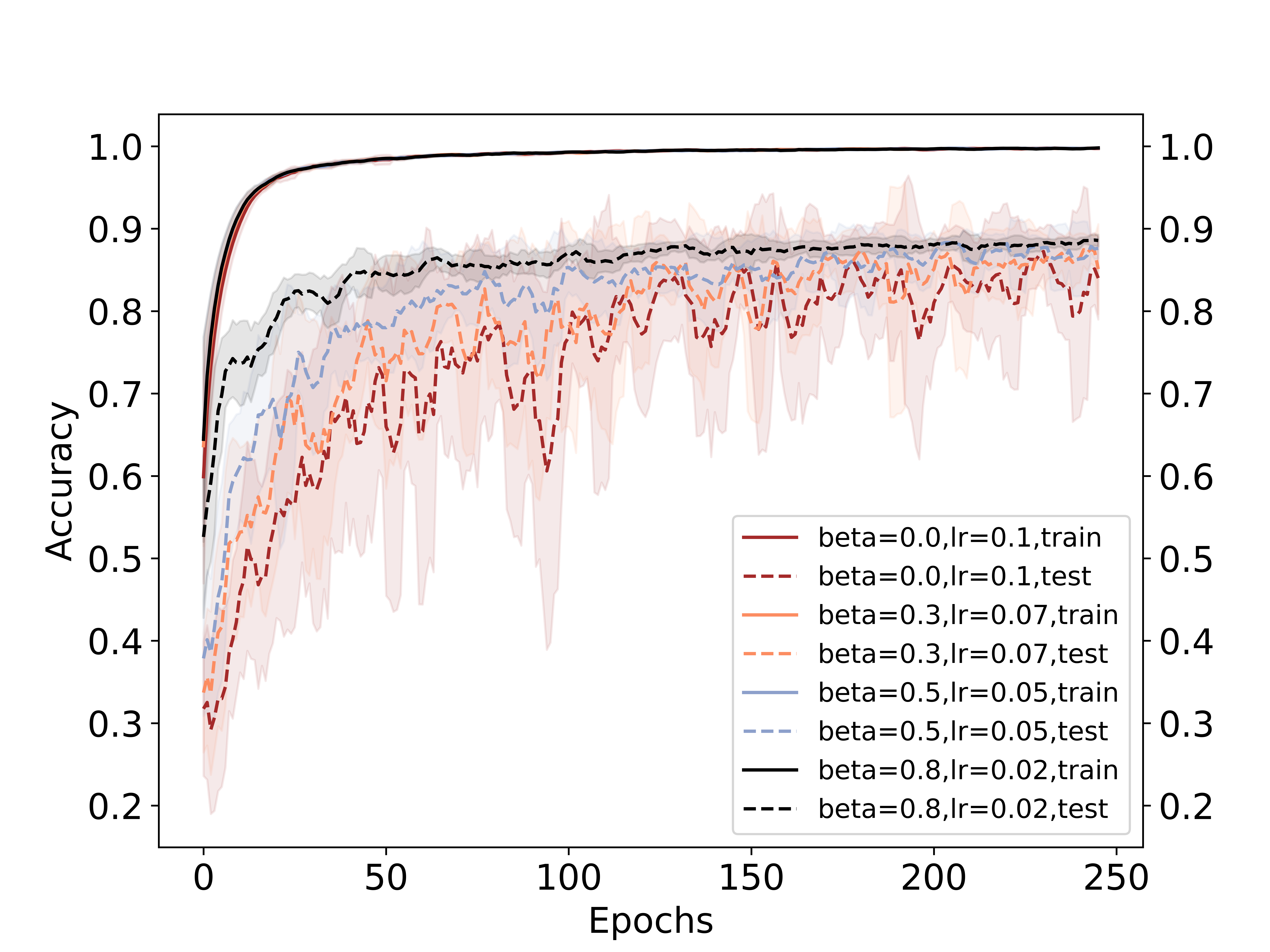}
            \caption[Densenet-121]%
            {{\small DenseNet-121}}    
            \label{fig:mean and std of net44}
        \end{subfigure}
        \caption[ The average and standard deviation of critical parameters ]
        {\small Classification results for CIFAR-10 dataset with various network architectures with combinations of $(h,\beta)$ chosen such that the effective learning rate $\frac{h}{(1-\beta)}$ remains same. In all of the experiments, external regularization like weight-decay, l.r scheduler, dropout,label-smoothing are kept off (except Batch-normalization). The results have been averaged over 3 random seeds having different initializations. (SGD+M) has a) higher test accuracy for increasing $\beta$ than (SGD) confirming Remark \ref{remark-sgd-1} b) Less variance for test accuracy confirming Remark \ref{remark-sgd-2}. }.  
        \label{SGD+Mfig}
    \end{figure*}
    
To study the effect of implicit regularization (SGD+M), a series of experiments have been performed on an image classification task. Four well-known and popular network architectures namely DenseNet \citep{iandola2014densenet}, ResNet-18, ResNet-50 \citep{he2016deep} and WideResNet \citep{zagoruyko2016wide} are trained to classify images from the CIFAR-10 and CIFAR-100 datasets. We are interested to know how well training these  networks with (SGD) and (SGD+M) respectively can generalize well onto the test dataset. To solely observe the effects of the momentum parameter $\beta$ and learning rate $h$ in generalization, we turn off all the external regularization  like dropout, weight-decay and label-smoothing. 
We fix the batch-size to 640 in all our experiments. 
\begin{table}[h]
\centering
\setlength\tabcolsep{2.5pt}
   \begin{tabular}{ c c c c c | c c c c}
\toprule
 & \multicolumn{4}{c}{CIFAR-10} &  \multicolumn{4}{c}{CIFAR-100} \\
 \hline
  $\beta$ /$h$ & DN-121 & RN-18 & RN-50 & WRN-16-8 & DN-121 & RN-18 & RN-50 & WRN-16-8   \\
 \hline
0.0/0.10&84.0$\pm$5.0&79.7$\pm$5.6&79.3$\pm$2.5&65.3$\pm$18.1&60.4$\pm$4.9&53.1$\pm$0.6&47.4$\pm$2.1&38.6$\pm$3.8\\
0.3/0.07&85.1$\pm$5.4&78.7$\pm$9.3&80.0$\pm$1.5&72.5$\pm$7.4&60.0$\pm$8.6&52.7$\pm$1.0&48.9$\pm$2.2&37.0$\pm$6.1\\
0.5/0.05&87.6$\pm$1.2&81.5$\pm$0.9&80.7$\pm$0.7&71.8$\pm$9.8&63.2$\pm$2.4&53.3$\pm$1.1&\textbf{50.3$\pm$1.0}&39.4$\pm$3.9\\
0.8/0.02&\textbf{88.6$\pm$0.7}&\textbf{82.4$\pm$0.4}&\textbf{82.4$\pm$0.7}&\textbf{75.4$\pm$2.8}&\textbf{64.7$\pm$0.8}&\textbf{54.3$\pm$0.6}&49.6$\pm$0.7&\textbf{40.6$\pm$1.2}\\
\bottomrule
\end{tabular}
\caption{\small Testing accuracy of CIFAR-10 and CIFAR-100 with different momentum $\beta$ and learning rates $h$, but the same effective learning rate $\frac{h}{(1-\beta)} = 0.1$. The best performance of different models is highlighted. The mean and the standard deviation of test accuracy is calculated over the last 5 epochs and three random seed initializations.  DN-121 is the Densenet-121, RN-18 and RN-50 denote the Resnet-18 and Resnet-50, WRN-16-8 represents the WideResnet with depth 16 and width-factor 8.  } \label{tab}
\end{table}

In the first experiment, we showed (GD+M) has a larger stability region than (GD) and hence allows for the use of a larger effective learning rate. The same conclusion holds for (SGD+M) and (SGD). However, here we want to show that even in the region where both algorithms are stable, (SGD+M) is still not just a scaled version of (SGD). For this purpose, we pick a small learning rate to ensure stability of both algorithms,  and keep the effective learning rate $\frac{h}{(1-\beta)}$ for (SGD+M) to be the same as the learning rate for (SGD) (both equal 0.1). 
We observe from Table \ref{tab} (also Figure \ref{SGD+Mfig}), that the maximum test accuracy is almost always achieved at the highest value for $\beta $. This observation is consistent with Remark-\ref{remark-sgd-2} where we showed that the implicit regularization in (SGD+M) is indeed stronger than (SGD), even after the learning rate adjustment.

The standard deviation of test accuracy in Table \ref{tab} is calculated over the last 5 epochs and three random seed initialization. Lower standard deviation indicates a smoother test accuracy curve meaning less variation of test accuracy within an epoch interval. We observe that the lowest standard deviation  is achieved at the highest value of $\beta$. Hence the observation that variance reduction effect is more prominent with higher $\beta$ is consistent with Remark \ref{remark-sgd-1}.

\section{Combined effects of IGR and noise injection}

Despite its close relation to  sharpness (Section  \ref{Sec:3}), the IGR term $\|\nabla E\|^2$ gets very weak and irrelevant as $\x^k$ approaches a local minimizer, since $\nabla E \rightarrow 0$. However, we find that this would not be the case if there was noise injection, which can help the IGR term retain its power even near local minima. More specifically, as studied in previous literature \citep{orvieto2022explicit, camuto2020explicit}, the algorithm resulting from injecting noise to each iteration of GD is usually called PGD (Perturbed gradient descent) that essentially minimizes an averaged objective function 
\[
R(\x)  := \mathbb{E}_{\eta\sim N(0,\sigma^2 \I)} E(\x+\etab).
\]
For small values of $\sigma$, we can expand $R(\x)$ into
\[
R(\x) = E(\x) + \mathbb{E}_{\eta} \etab^T\nabla E(\x) +\frac{1}{2}\mathbb{E}_{\eta} \etab^T \nabla^2 E(\x) \etab + O(\sigma^3) =  E(\x) +\frac{1}{2} \sigma^2 Tr(\nabla^2 E(\x))  +O(\sigma^3),
\]
where $Tr$ denotes the trace operator. Thus minimizing $R(\x)$ regularizes the trace Hessian of $E$. When minimizing $R(\x)$ using an SGD type of update, the iterations would be
\[
\x^{k+1} = \x^{k} -h \nabla E(\x^{k} +\etab_k), \ \ \textrm{where} \ \ \etab_k \sim \mathcal{N}(\mathbf{0},\sigma^2 \I),
\]
which is known as a form of PGD.
Because of the finite learning rate, the updates would follow the modified flow with an IGR term, which in this case is
\begin{equation*}\label{eq:obj_noise}
E(\x+\etab_k) +  \frac{h}{4}\|\nabla E(\x+\etab_k)\|^2.
\end{equation*}
In expectation, the modified loss is
\begin{align*}
& \mathbb{E}_{\etab_k} \left[E(\x+\etab_k) +  \frac{h}{4}\|\nabla E(\x+\etab_k)\|^2)\right] \\  = &\mathbb{E}_{\etab_k} \left[E(\x) +  \etab_k^T \nabla E(\x)+ \frac{1}{2}(\etab_k)^T \nabla^2 E(\x)\etab_k \right]  \\ & + \frac{h}{4} \mathbb{E}_{\etab_k} \|\nabla E(\x) + \nabla^2 E(\x) \etab_k+ \nabla^3 E(\x)[\etab_k,\etab_k] \|^2)+O(\sigma^3)  \\
 = &E(\x) + \sigma^2 Tr(\nabla^2 E(\x)) + \frac{h}{4}\left[\sigma^2\|\nabla^2 E(\x)\|_F^2+ \mathbb{E}_{\etab_k}\|\nabla E(\x)+\nabla^3 E(\x)[\etab_k,\etab_k]\|^2   \right]+  O(\sigma^3).
\end{align*}
We see that now there is a Hessian regularization term $\frac{h}{4}\sigma^2\|\nabla^2 E(\x)\|_F^2$ coming out of IGR which does not vanish even around local minimizers, and it's strength is proportional to the learning rate. We expect this new regularization term to get stronger when momentum is added, as momentum amplifies the power of IGR (Remark \ref{remark-gd-2}, \ref{remark-sgd-2}).  This observation suggests that IGR and noise injection as two different types of implicit regularization might be able to reinforce each other when used collaboratively. 

\section{Conclusion} 
This work studies the generalization of momentum driven gradient descent approach through the lens of implicit regularization (IR) with both theoretical analysis and experimental validation provided. We examined the similarities and differences between (SGD) and (SGD+M) and find that (SGD+M) with suitable parameters outperforms (SGD) in almost all settings. Moreover, we found that in addition to momentum, IGR may also be magnified by noise injection, which is a topic we want to further explore in the future. %


\bibliography{iclr2023_conference}

\begin{thebibliography}{34}
\providecommand{\natexlab}[1]{#1}
\providecommand{\url}[1]{\texttt{#1}}
\expandafter\ifx\csname urlstyle\endcsname\relax
  \providecommand{\doi}[1]{doi: #1}\else
  \providecommand{\doi}{doi: \begingroup \urlstyle{rm}\Url}\fi

\bibitem[Arora et~al.(2019)Arora, Cohen, Hu, and Luo]{arora2019implicit}
Sanjeev Arora, Nadav Cohen, Wei Hu, and Yuping Luo.
\newblock Implicit regularization in deep matrix factorization.
\newblock \emph{Advances in Neural Information Processing Systems}, 32, 2019.

\bibitem[Barrett \& Dherin(2020)Barrett and Dherin]{barrett2020implicit}
David~GT Barrett and Benoit Dherin.
\newblock Implicit gradient regularization.
\newblock \emph{arXiv preprint arXiv:2009.11162}, 2020.

\bibitem[Bengio \& LeCun(2007)Bengio and LeCun]{Bengio+chapter2007}
Yoshua Bengio and Yann LeCun.
\newblock Scaling learning algorithms towards {AI}.
\newblock In \emph{Large Scale Kernel Machines}. MIT Press, 2007.

\bibitem[Camuto et~al.(2020)Camuto, Willetts, Simsekli, Roberts, and
  Holmes]{camuto2020explicit}
Alexander Camuto, Matthew Willetts, Umut Simsekli, Stephen~J Roberts, and
  Chris~C Holmes.
\newblock Explicit regularisation in gaussian noise injections.
\newblock \emph{Advances in Neural Information Processing Systems},
  33:\penalty0 16603--16614, 2020.

\bibitem[Cohen et~al.(2021)Cohen, Kaur, Li, Kolter, and
  Talwalkar]{cohen2021gradient}
Jeremy~M Cohen, Simran Kaur, Yuanzhi Li, J~Zico Kolter, and Ameet Talwalkar.
\newblock Gradient descent on neural networks typically occurs at the edge of
  stability.
\newblock \emph{arXiv preprint arXiv:2103.00065}, 2021.

\bibitem[Foret et~al.(2020)Foret, Kleiner, Mobahi, and
  Neyshabur]{foret2020sharpness}
Pierre Foret, Ariel Kleiner, Hossein Mobahi, and Behnam Neyshabur.
\newblock Sharpness-aware minimization for efficiently improving
  generalization.
\newblock \emph{arXiv preprint arXiv:2010.01412}, 2020.

\bibitem[Goodfellow et~al.(2016)Goodfellow, Bengio, Courville, and
  Bengio]{goodfellow2016deep}
Ian Goodfellow, Yoshua Bengio, Aaron Courville, and Yoshua Bengio.
\newblock \emph{Deep learning}, volume~1.
\newblock MIT Press, 2016.

\bibitem[Gunasekar et~al.(2017)Gunasekar, Woodworth, Bhojanapalli, Neyshabur,
  and Srebro]{gunasekar2017implicit}
Suriya Gunasekar, Blake~E Woodworth, Srinadh Bhojanapalli, Behnam Neyshabur,
  and Nati Srebro.
\newblock Implicit regularization in matrix factorization.
\newblock \emph{Advances in Neural Information Processing Systems}, 30, 2017.

\bibitem[He et~al.(2016)He, Zhang, Ren, and Sun]{he2016deep}
Kaiming He, Xiangyu Zhang, Shaoqing Ren, and Jian Sun.
\newblock Deep residual learning for image recognition.
\newblock In \emph{Proceedings of the IEEE conference on computer vision and
  pattern recognition}, pp.\  770--778, 2016.

\bibitem[Hinton et~al.(2006)Hinton, Osindero, and Teh]{Hinton06}
Geoffrey~E. Hinton, Simon Osindero, and Yee~Whye Teh.
\newblock A fast learning algorithm for deep belief nets.
\newblock \emph{Neural Computation}, 18:\penalty0 1527--1554, 2006.

\bibitem[Iandola et~al.(2014)Iandola, Moskewicz, Karayev, Girshick, Darrell,
  and Keutzer]{iandola2014densenet}
Forrest Iandola, Matt Moskewicz, Sergey Karayev, Ross Girshick, Trevor Darrell,
  and Kurt Keutzer.
\newblock Densenet: Implementing efficient convnet descriptor pyramids.
\newblock \emph{arXiv preprint arXiv:1404.1869}, 2014.

\bibitem[Ibayashi \& Imaizumi(2022)Ibayashi and
  Imaizumi]{ibayashi2022quasipotential}
Hikaru Ibayashi and Masaaki Imaizumi.
\newblock Quasi-potential theory for escape problem: Quantitative sharpness
  effect on {SGD}'s escape from local minima, 2022.
\newblock URL \url{https://openreview.net/forum?id=vLz0e9S-iF3}.

\bibitem[Jelassi \& Li(2022)Jelassi and Li]{jelassi2022towards}
Samy Jelassi and Yuanzhi Li.
\newblock Towards understanding how momentum improves generalization in deep
  learning.
\newblock In \emph{International Conference on Machine Learning}, pp.\
  9965--10040. PMLR, 2022.

\bibitem[Ji \& Telgarsky(2019)Ji and Telgarsky]{ji2019implicit}
Ziwei Ji and Matus Telgarsky.
\newblock The implicit bias of gradient descent on nonseparable data.
\newblock In \emph{Conference on Learning Theory}, pp.\  1772--1798. PMLR,
  2019.

\bibitem[Kingma \& Ba(2014)Kingma and Ba]{kingma2014adam}
Diederik~P Kingma and Jimmy Ba.
\newblock Adam: A method for stochastic optimization.
\newblock \emph{arXiv preprint arXiv:1412.6980}, 2014.

\bibitem[Kovachki \& Stuart(2021)Kovachki and Stuart]{kovachki2021continuous}
Nikola~B Kovachki and Andrew~M Stuart.
\newblock Continuous time analysis of momentum methods.
\newblock \emph{Journal of Machine Learning Research}, 22\penalty0
  (17):\penalty0 1--40, 2021.

\bibitem[Li et~al.(2017)Li, Tai, and Weinan]{li2017stochastic}
Qianxiao Li, Cheng Tai, and E~Weinan.
\newblock Stochastic modified equations and adaptive stochastic gradient
  algorithms.
\newblock In \emph{International Conference on Machine Learning}, pp.\
  2101--2110. PMLR, 2017.

\bibitem[Li et~al.(2019)Li, Tai, and E]{JMLR:v20:17-526}
Qianxiao Li, Cheng Tai, and Weinan E.
\newblock Stochastic modified equations and dynamics of stochastic gradient
  algorithms i: Mathematical foundations.
\newblock \emph{Journal of Machine Learning Research}, 20\penalty0
  (40):\penalty0 1--47, 2019.
\newblock URL \url{http://jmlr.org/papers/v20/17-526.html}.

\bibitem[Neyshabur(2017)]{neyshabur2017implicit}
Behnam Neyshabur.
\newblock Implicit regularization in deep learning.
\newblock \emph{arXiv preprint arXiv:1709.01953}, 2017.

\bibitem[Neyshabur et~al.(2014)Neyshabur, Tomioka, and
  Srebro]{neyshabur2014search}
Behnam Neyshabur, Ryota Tomioka, and Nathan Srebro.
\newblock In search of the real inductive bias: On the role of implicit
  regularization in deep learning.
\newblock \emph{arXiv preprint arXiv:1412.6614}, 2014.

\bibitem[Orvieto et~al.(2022{\natexlab{a}})Orvieto, Kersting, Proske, Bach, and
  Lucchi]{orvieto2022anticorrelated}
Antonio Orvieto, Hans Kersting, Frank Proske, Francis Bach, and Aurelien
  Lucchi.
\newblock Anticorrelated noise injection for improved generalization.
\newblock \emph{arXiv preprint arXiv:2202.02831}, 2022{\natexlab{a}}.

\bibitem[Orvieto et~al.(2022{\natexlab{b}})Orvieto, Raj, Kersting, and
  Bach]{orvieto2022explicit}
Antonio Orvieto, Anant Raj, Hans Kersting, and Francis Bach.
\newblock Explicit regularization in overparametrized models via noise
  injection.
\newblock \emph{arXiv preprint arXiv:2206.04613}, 2022{\natexlab{b}}.

\bibitem[Poggio et~al.(2020)Poggio, Banburski, and Liao]{poggio2020theoretical}
Tomaso Poggio, Andrzej Banburski, and Qianli Liao.
\newblock Theoretical issues in deep networks.
\newblock \emph{Proceedings of the National Academy of Sciences}, 117\penalty0
  (48):\penalty0 30039--30045, 2020.

\bibitem[Polyak(1964)]{polyak1964some}
Boris~T Polyak.
\newblock Some methods of speeding up the convergence of iteration methods.
\newblock \emph{Ussr computational mathematics and mathematical physics},
  4\penalty0 (5):\penalty0 1--17, 1964.

\bibitem[Razin \& Cohen(2020)Razin and Cohen]{razin2020implicit}
Noam Razin and Nadav Cohen.
\newblock Implicit regularization in deep learning may not be explainable by
  norms.
\newblock \emph{Advances in neural information processing systems},
  33:\penalty0 21174--21187, 2020.

\bibitem[Smith \& Le(2017)Smith and Le]{smith2017bayesian}
Samuel~L Smith and Quoc~V Le.
\newblock A bayesian perspective on generalization and stochastic gradient
  descent.
\newblock \emph{arXiv preprint arXiv:1710.06451}, 2017.

\bibitem[Smith et~al.(2021)Smith, Dherin, Barrett, and De]{smith2021origin}
Samuel~L Smith, Benoit Dherin, David~GT Barrett, and Soham De.
\newblock On the origin of implicit regularization in stochastic gradient
  descent.
\newblock \emph{arXiv preprint arXiv:2101.12176}, 2021.

\bibitem[Soudry et~al.(2018)Soudry, Hoffer, Nacson, Gunasekar, and
  Srebro]{soudry2018implicit}
Daniel Soudry, Elad Hoffer, Mor~Shpigel Nacson, Suriya Gunasekar, and Nathan
  Srebro.
\newblock The implicit bias of gradient descent on separable data.
\newblock \emph{The Journal of Machine Learning Research}, 19\penalty0
  (1):\penalty0 2822--2878, 2018.

\bibitem[Sutskever et~al.(2013)Sutskever, Martens, Dahl, and
  Hinton]{sutskever2013importance}
Ilya Sutskever, James Martens, George Dahl, and Geoffrey Hinton.
\newblock On the importance of initialization and momentum in deep learning.
\newblock In \emph{International conference on machine learning}, pp.\
  1139--1147. PMLR, 2013.

\bibitem[Tieleman et~al.(2012)Tieleman, Hinton, et~al.]{tieleman2012lecture}
Tijmen Tieleman, Geoffrey Hinton, et~al.
\newblock Lecture 6.5-rmsprop: Divide the gradient by a running average of its
  recent magnitude.
\newblock \emph{COURSERA: Neural networks for machine learning}, 4\penalty0
  (2):\penalty0 26--31, 2012.

\bibitem[Vardi \& Shamir(2021)Vardi and Shamir]{vardi2021implicit}
Gal Vardi and Ohad Shamir.
\newblock Implicit regularization in relu networks with the square loss.
\newblock In \emph{Conference on Learning Theory}, pp.\  4224--4258. PMLR,
  2021.

\bibitem[Wang et~al.(2021)Wang, Meng, Zhang, Sun, Chen, Ma, and
  Liu]{https://doi.org/10.48550/arxiv.2110.03891}
Bohan Wang, Qi~Meng, Huishuai Zhang, Ruoyu Sun, Wei Chen, Zhi-Ming Ma, and
  Tie-Yan Liu.
\newblock Does momentum change the implicit regularization on separable data?,
  2021.
\newblock URL \url{https://arxiv.org/abs/2110.03891}.

\bibitem[Wu et~al.(2020)Wu, Hu, Xiong, Huan, Braverman, and Zhu]{wu2020noisy}
Jingfeng Wu, Wenqing Hu, Haoyi Xiong, Jun Huan, Vladimir Braverman, and
  Zhanxing Zhu.
\newblock On the noisy gradient descent that generalizes as sgd.
\newblock In \emph{International Conference on Machine Learning}, pp.\
  10367--10376. PMLR, 2020.

\bibitem[Zagoruyko \& Komodakis(2016)Zagoruyko and
  Komodakis]{zagoruyko2016wide}
Sergey Zagoruyko and Nikos Komodakis.
\newblock Wide residual networks.
\newblock \emph{arXiv preprint arXiv:1605.07146}, 2016.

\end{thebibliography}
\bibliographystyle{iclr2023_conference}
\end{document}


\maketitle
This appendix is organized in the following manner. In Section \ref{sec:1}, we prove Theorem 5.1 for (SGD+M) in the main paper by showing the following two parts: 
\begin{itemize}
    \item In Theorem \ref{thm:main_theorem} we prove that the local error between the continuous piecewise-differentiable trajectory and the Heavy Ball momentum SGD in each iteration is $O(h^3)$. This Theorem depends on Lemma \ref{lm:bound_derivative} which shows that higher order derivatives are bounded. 
    \item Then in Theorem \ref{thm:global} we prove that the global error between the continuous trajectory and the Heavy-ball momentum update is $O(h^2)$. 
\end{itemize}
In Section \ref{sec:2}, we prove Theorem 4.1 in the main paper for (GD+M) by showing that this is sub-case of Theorem 5.1 (in the main paper) when the mini-batch loss $E_{k}$ is the full-batch loss $E$ for each of the piece-wise differentiable trajectory in Corollary \ref{cor:GD+M} (appendix). 
Finally, we presented the proof for Remark-5.2 and Remark-5.3 in Section \ref{sec:4} and Section \ref{sec:3}.

\section{Proof of Theorem 5.1 }
\label{sec:1}
Throughout the appendix, $\|\cdot\|$ represents $\ell_2$ norm for vectors and matrices, and Frobenius norm for 3 or 4-dimensional tensors.
\begin{theorem}
\label{thm:main_theorem}[Bound on the local error]
Let the loss for each mini-batch $E_{n}$ be smooth and sufficiently (4-times) differentiable, and its zeroth to fourth order derivatives are bounded, then in each iteration, the Heavy Ball momentum SGD update
\begin{equation}
\left\{
\begin{aligned}\label{eq:SGD-HB}
&\x^{n+1} = \x^{n} - h \nabla E_{n}(\x^{n}) + \beta(\x^{n} -\x^{n-1}), & &  {n = 1,2,...,N}\\
&\x^{1} = \x^{0}-  h \nabla E_{0}(\x^{0})\\
&\x^{0} = \x^{-1}=\mathbf{0} \\ 
\end{aligned}
\right.
\end{equation} 
is locally $O(h^3)$-close to the flow of the following modified ODE when updating the $n^{th}$ mini-batch 
\begin{align}
\label{eq:main_scheme}
    \W\x{'} (t) = - \nabla G_{n}(\W\x(t)) - A_n(\W\x(t)),  &  & \text{ for $t_{n} \leq t < t_{n+1}$}
\end{align}
with $t_n = nh$.
Specifically, the following equality holds for each iteration
\begin{equation}\label{eq:hcubeerr}
    \W\x(t_{n+1}) = \W\x(t_{n}) - h\nabla E_{n}(\W\x(t_{n}))  + \beta( \W\x(t_{n}) -\W\x(t_{n-1}) ) + O(h^3). 
\end{equation}
Here 
\[G_{n}(\W\x) = \sum_{k=0}^n \beta^{n-k} E_{k}(\W\x) , A_n( \W\x) = \frac{h}{2} \sum_{k=0}^n \beta^{n-k}  C_{k} (\W\x),
\]
and $C_{k}(\W\x) = \nabla^2 G_{k}(\W\x) \nabla G_{k}(\W\x) + \beta  \nabla^2 G_{k-1}(\W\x) \nabla G_{k-1}(\W\x))$ with initial condition $C_{0}(\x) =\nabla^2 G_{0}(\W\x) \nabla G_{0}(\W\x)$. 
\end{theorem}

\begin{proof}
Before proceeding with the proof, we state that the continuous trajectory $\W\x(t)$ is differentiable in the whole domain $[0,T]$ except at the grid points $t_{n}, n=0,1,...,N$, $N = \lfloor \frac{T}{h} \rfloor$. This is because in the interval $t_{n} \leq t < t_{n+1}$, the continuous trajectory is given by  $\W\x{'} (t) = - \nabla G_{n}(\W\x(t)) - A_n(\W\x(t))$, whereas in the the interval $t_{n-1} \leq t < t_{n}$, the trajectory is defined by $\W\x{'} (t) = - \nabla G_{n-1}(\W\x(t)) - A_{n-1}(\W\x(t))$. We notice the following:
\begin{itemize}
    \item the right-hand side and the left-hand side derivatives of the trajectories at any boundary point $t_{n}$ are not equal, i.e, $\W\x{'} (t_{n}^{+}) \neq \W\x{'} (t_{n}^{-}),  n=0,1,...,N$. {Here we define $\W\x'(t_0^-)=\bold{0}$.}
    \item based on Lemma \ref{lm:bound_derivative}, the norms of the first to third-order derivatives of the trajectory $\W\x(t)$, $t\in [0,T]$ can be bounded by constants independent of $h$. 
    Hence from \ref{eq:main_scheme} we can compute the left and right side derivatives at $t_k$  (any boundary point $\forall{k=1,2,...,n})$ as follows: 
\end{itemize}

\begin{enumerate}
\label{simplcify}
    \item $  \W\x{'} (t_{n}^{+}) = - \nabla G_{n}(\W\x(t_{n})) - A_n(\W\x(t_n)) $ as it belongs to trajectory  $t_{n} \leq t < t_{n+1}$
    \item  $\W\x{''} (t_{n}^{+}) =  \nabla^2 G_{n}(\W\x(t_{n}))\nabla G_{n}(\W\x(t_{n})) + O(h)$
      \item $  \W\x{'} (t_{n}^{-}) = - \nabla G_{n-1}(\W\x(t_{n})) - A_{n-1}(\W\x(t_n)) $ as it belongs to trajectory  $t_{n-1} \leq t < t_{n}$
    \item  $\W\x{''} (t_{n}^{-}) =  \nabla^2 G_{n-1}(\W\x(t_{n}))\nabla G_{n-1}(\W\x(t_{n})) + O(h)$
\end{enumerate}
As $\|\W\x{'''}(t)\|$ is bounded by a constant (Lemma \ref{lm:bound_derivative}), the Taylor expansion of the trajectory $\W\x(t)$ is done at the point $t_{n}$ on both sides:
\begin{align}
\label{eq:taylor}
    &\W \x(t_{n+1})=\W\x(t_n)+\W\x'(t_n^+)h+\frac{h^2}{2}\W\x''(t_n^+)+O(h^3),\\
&\W\x(t_n)=\W\x(t_{n-1})+h\W\x'(t_n^-)-\frac{h^2}{2}\W\x''(t_n^-)+O(h^3).
\end{align}
Recall that the main objective of this theorem is to find how well the continuous trajectory $\W\x(t)$ satisfies the H.B  momentum update equation \ref{eq:SGD-HB}. Hence we plug  $\W\x(t_n)$ into \ref{eq:SGD-HB} and examine the resulting error, which is also known as the Local Truncation Error (LTE) in the numerical ODE literature. The calculation is carried out as follows
\begin{align*}
     &\W\x(t_{n+1})-\W\x(t_n)-\beta(\W\x(t_n)-\W\x(t_{n-1}))\\
    =\ \ & h \W\x'(t_n^+)+\frac{h^2}{2}\W\x''(t_n^+)-\beta(h \W\x'(t_n^-)-\frac{h^2}{2}\W\x''(t_n^-)) + O(h^3)\\
    =\ \ & -h\nabla G_n(\W\x(t_n))-h A_n(\W\x(t_n))+\frac{h^2}{2}\nabla^2 G_n(\W\x(t_n))\nabla G_n(\W\x(t_n))+\beta h\nabla G_{n-1}(\W\x(t_n))+\beta h A_{n-1}(\W\x(t_n))\\
    &+\beta\frac{h^2}{2}\nabla^2 G_{n-1}(\W\x(t_n))\nabla G_{n-1}(\W\x(t_n))+O(h^3)\\
    =\ \ &-h\underbrace{(\nabla G_n(\W\x(t_n))-\beta \nabla G_{n-1}(\W\x(t_n)))}_{\nabla E_n(\W\x(t_n))} -h\underbrace{(A_n(\W\x(t_n))-\beta A_{n-1}(\W\x(t_n)))}_{\frac{h}{2}C_{n} }\\
    &+\frac{h^2}{2}\underbrace{(\nabla^2 G_n(\W\x(t_n))\nabla G_n(\W\x(t_n))+\beta\nabla^2 G_{n-1}(\W\x(t_n))\nabla G_{n-1}(\W\x(t_n)))}_{C_{n} \text{(by definition)}}+O(h^3)\\
    =\ \ &-h\nabla E_n(\W\x(t_n))-h\frac{h}{2}C_n(\W\x(t_n))+\frac{h^2}{2}C_n(\W\x(t_n))+O(h^3)\\
    =\ \ &-h\nabla E_n(\W\x(t_n))+O(h^3).
\end{align*}
The equation has been simplified using the following two identities which is easy to verify: $$\nabla G_n(\W\x(t_n))-\beta \nabla G_{n-1}(\W\x(t_n))) =  E_n(\W\x(t_n))  $$ and $$A_n(\W\x(t_n))-\beta A_{n-1}(\W\x(t_n)) = \frac{h}{2}C_{n}. $$ Hence we proved that the solution of the continuous trajectory satisfies the discrete H.B. momentum updates up to an error of order $O(h^3)$. 

\end{proof}

\begin{corollary}
The piecewise ODE \eqref{eq:SGD-HB} in Theorem \ref{thm:main_theorem}, can be equivalently written as:
\begin{align}\label{ode_main_modified}
\begin{split}
   & \W \x{'} (t) = - \nabla \hat{E}_n(\W\x(t))  \quad  \text{ for $t_{n} \leq t < t_{n+1}$},  \\
   & \text{where,} \quad  \hat{E}_n( \W \x) = G_{n}( \W \x) + \frac{h}{4} ( \|\nabla G_{n}( \W \x)\|^2 + 2 \sum_{r=0}^{n-1} \beta^{n-r} \|\nabla G_{r}( \W \x)\|^2 ),
   \end{split}
\end{align}
which is what we used in the statement of Theorem 5.1 in the main paper.
\end{corollary}

\begin{proof}
By the definition of the ODE in Theorem \ref{thm:main_theorem}, we can see
\begin{align*}
A_n(\W\x(t))&=\frac{h}{2}\sum_{k=1}^n\beta^{n-k}\left(\nabla^2 G_k(\W\x)\nabla G_k(\W\x)+\beta\nabla^2 G_{k-1}(\W\x)\nabla G_{k-1}(\W\x)\right) + \frac{h}{2}\beta^n\nabla^2 G_0(\W\x)\nabla G_k(\W\x)\\
&=\frac{h}{2}\sum_{k=0}^n\beta^{n-k}\nabla^2 G_k(\W\x(t))\nabla G_k(\W\x(t))+\frac{h}{2}\sum_{k=1}^n\beta^{n-k+1}\nabla^2 G_{k-1}(\W\x(t))\nabla G_{k-1}(\W\x(t))\\
&=\frac{h}{2}\nabla^2 G_n(\W\x(t))\nabla G_n(\W\x(t))+h\sum_{k=0}^{n-1}\beta^{n-k}\nabla^2 G_k(\W\x(t))\nabla G_k(\W\x(t)).
\end{align*}
Therefore we can rewrite \eqref{eq:main_scheme} as
\begin{equation}
\label{ode_mod}
    \W\x'(t)=-\nabla G_n(\W\x(t))-\frac{h}{4}\nabla\left(\|\nabla G_n(\W\x(t))\|^2+2\sum_{k=0}^{n-1}\beta^{n-k}\|\nabla G_k(\W\x(t))\|^2\right) := \nabla\hat E_n(\W\x(t)),
\end{equation}
where $\hat E_n(\W\x)=G_n(\W\x)+\frac{h}{4}\left(\|\nabla G_n(\W\x)\|^2+2\sum_{k=0}^{n-1}\beta^{n-k}\|\nabla G_k(\W\x)\|^2\right)$.
\end{proof}

 \begin{lemma}\label{lm:bound_derivative}
Under the assumption of Theorem \ref{thm:main_theorem}, let $\W\x(t)$ be defined as in \eqref{eq:main_scheme}, then the first to third order derivatives of $\W\x$ with respect to time are bounded, i.e., there exists constants $c_1,\ c_2,\ c_3$ such that $\|\W\x'(t)\|\leq c_1, \|\W\x''(t)\|\leq c_2,\ \|\W\x'''(t)\|\leq c_3,$ for all $t\in[0,T]$. 
\end{lemma}

\begin{proof}
 Although the continuous trajectory $\W\x(t)$ is defined piece-wise, with the a step-size $h$, we want to obtain a constant (i.e., $h$-independent) upper bound for its derivatives so as to faithfully truncate the Taylor-expansion \eqref{eq:taylor} to get an error of $O(h^3)$. By the assumption of Theorem \ref{thm:main_theorem}, we have boundedness of $\|\nabla^{(\alpha)} E_{k}(\x)\|$ for $0\leq\alpha \leq4$, i.e., with some constant $c_{0}$, for all $1\leq n\leq N$,
 \[\sup\|\nabla^{(\alpha)} E_{n}(\x)\| \leq c_0, \quad 0\leq\alpha\leq 4, 
 \]
 where $\|\cdot\|$ denotes the Frobenius norm of the tensors. 
 
 Then we can immediately bound the derivatives of $G_k$ and $C_k$, for  any $k$, $\x$, and $0\leq \alpha \leq 4$,
 
\begin{align}
\label{gkbound}
    \|\nabla^{(\alpha)} G_k(\W\x)\|=\left\|\sum_{k=0}^n \beta^{n-k}\nabla^{(\alpha)} E_k(\W\x)\right\|\leq \sum_{k=0}^n\beta^{n-k} \|\nabla^{(\alpha)} E_k(\W \x)\|\leq \frac{c_0}{1-\beta},
\end{align}

\begin{align}
\label{ckbound}
    \|C_k(\W\x)\|\leq \|\nabla^2 G_k(\W\x)\|\|\nabla G_k(\W\x)\|+\beta\|\nabla^2 G_{k-1}(\W\x)\|\|\nabla G_{k-1}(\W\x)\|\leq \frac{1+\beta}{(1-\beta)^2}c_0^2.
\end{align}
Now, using \eqref{gkbound} and \eqref{ckbound} and the boundedness of $h$ (i.e., $h\leq T$), we can show that for some constant $c_{1}$,
\begin{align}\label{eq:bound_first}
    \begin{split}
    \|\W\x'(t)\|&\leq \|G_n(\W\x(t))\|+\frac{h}{2}\left\|\sum_{k=0}^n\beta^{n-k}C_k(\W\x)\right\|\\
    &\leq \sum_{k=0}^n \beta^{n-k}\max_{\W\x,k}\|E_k(\W\x)\| +\frac{h}{2}\sum_{k=0}^n\beta^{n-k} \max_{\W\x,k}\|C_k(\W\x)\|\\
    &\leq\frac{c_0}{(1-\beta)^2}+\frac{c_0^2h(1+\beta)}{2(1-\beta)^3}:=c_1.
    \end{split}
\end{align}
Next we show that $\|\W\x''(t)\|$ is uniformly bounded, 
 \begin{align*}
    \W \x''(t) &= -\sum_{k=0}^n \beta^{n-k}\nabla^2 E_k(\W \x(t))\W\x'(t) \\
    & \quad - \frac{h}{2}\sum_{k=0}^n \beta^{n-k}\underbrace{\left(\nabla^3 G_k(\W\x(t))[\W\x'(t)]\nabla G_k(\W x(t))+\nabla^2 G_k(\W\x(t))\nabla^2 G_k(\W \x(t))\W\x'(t)\right)}_{(I)}\\
    &\quad-\frac{h}{2}\beta\sum_{k=0}^n \beta^{n-k}\underbrace{\left(\nabla^3 G_{k-1}(\W\x(t))[\W\x'(t)]\nabla G_{k-1}(\W \x(t))+\nabla^2 G_{k-1}(\W\x(t))\nabla^2 G_{k-1}(\W \x(t))\W\x'(t)\right)}_{(II)}.
\end{align*}
Hence we have:
 \begin{align}
 \label{intrm}
    & \| \W \x''(t) \| \leq \| \sum_{k=0}^n \beta^{n-k}\nabla^2 E_k(\W \x(t))\W\x'(t) \| + \frac{h}{2(1-\beta)} \|(I)\| + \frac{h\beta}{2(1-\beta)} \|(II)\|.
\end{align}
Individually examining $ \|(I)\|$ and $ \|(II)\|$, we have
\begin{align*}
    \|(I)\| &\leq \| \nabla^3 G_k(\W\x(t))[\W\x'(t)]\nabla G_k(\W \x(t))\|+ \|\nabla^2 G_k(\W\x(t))\nabla^2 G_k(\W \x(t))\W\x'(t)\|\\
    & \leq \|\nabla^3 G_k(\W\x(t)) \| \|\W\x'(t) \| \| \nabla G_k(\W \x(t)) \| + \| \nabla^2 G_k(\W\x(t)) \|^{2}  \|\W\x'(t) \|   \leq 2 c_{0}^2 c_{1}.
\end{align*}
Here in the last inequality we used the fact that $\|\W\x'\|\leq c_1$ as in \eqref{eq:bound_first}.
 Similarly, $\| (II)\|  \leq  2 c_{0}^2 c_{1} $.
 Putting these inequalities into \ref{intrm}, we have for some constant $c_{2}$:
 \begin{align*}
    &   \| \W \x''(t) \| \leq \| \sum_{k=0}^n \beta^{n-k}\nabla^2 E_k(\W \x(t))\W\x'(t) \| + \frac{h}{2(1-\beta)} \|(I)\| + \frac{h\beta}{2(1-\beta)} \|(II)\|\leq c_{2}.
 \end{align*}
 Finally, we bound the thrid order derivative 
 \begin{align*}
       \W\x'''(t) &=  -\sum_{k=0}^n \beta^{n-k}\nabla^3 E_k(\W\x(t))[\W\x'(t)]\W \x'(t) -\sum_{k=0}^n\beta^{n-k}\nabla^2 E_k(\W\x(t))\W\x''(t) \\ 
     & \quad- \frac{h}{2}\sum_{k=0}^n \beta^{n-k} \frac{d (I)}{d t } - \frac{h}{2}\beta \sum_{k=0}^n  \beta^{n-k} \frac{d (II)_{}}{d t }.
 \end{align*}
Bounding the norm on $  \W\x'''(t)$ based on this expression is straightforward.

\end{proof}

\begin{theorem}[Bound on the global error]\label{thm:global}
  Let $\W\x(t)$ be the solution to \eqref{eq:main_scheme} and assume the conditions in Theorem \ref{thm:main_theorem} hold. Then the global error $\|\e^{n}\| = \| \W\x(t_{n}) - \x^{n} \| $ is of order $O(h^2)$, where $\W \x(t_{n}) $ is the solution of the ODE in \eqref{eq:main_scheme} at the $n^{th}$ boundary point  and $\x^{n} $ is the discrete H.B Momentum update. 
\end{theorem}

\begin{proof}
In Theorem 1.1, we already showed that the solution of the piecewise ODE $ \W\x(t_{n}) $ satisfies 
\begin{align}
    \W\x(t_{n+1}) = \W\x(t_{n}) - h\nabla E_{n}(\W\x(t_{n}))  + \beta( \W\x(t_{n}) -\W\x(t_{n-1}) ) + O(h^3),
\end{align}
and by definition, the discrete H.B momentum update satisfy 
\begin{align}
    \x^{n+1} = \x^{n} - h \nabla E_{k}(\x^{n}) + \beta(\x^{n} -\x^{n-1}).
\end{align}
Let the error at the $n^{th}$ update is denoted by $\e_{n}=  \W\x(t_{n}) -  \x_{n}$, then subtracting the two updates, we have:
\begin{align}
\label{eq:error}
    & \e_{n+1} = \e_{n} + \beta ( \e_{n} -\e_{n-1} ) - h\underbrace{(\nabla E_{n}(\x^{n})- \nabla E_{n}(\W\x(t_{n})))}_{M} +O(h^3).
\end{align}
By the assumption in Theorem \ref{thm:main_theorem}, there exists some constant $c_1$ such that $\|\nabla^2 E_n\|_\infty:=\max_{\x} \|\nabla^2 E_n(\x)\|\leq c_1$. Then $M$ can be bounded by
\begin{align}
    & \| M\| = \|\nabla E_{n}(\x^{n})- \nabla E_{n}(\W\x(t_{n}))  \| \leq \|\nabla^2 E_n\|_\infty\|\e_{n} \| \leq c_1 \|  \e_{n} \|.
\end{align}

Now  taking the norm on $\e_{n+1}-\e_{n}$ and applying triangular inequality on the right hand side of \eqref{eq:error}, we have for some constant $c$:
\begin{align}\label{eq:iteration_gap}
  &  \| \e_{n+1}-\e_{n}\| \leq \beta  \| \e_{n}-\e_{n-1}\| +h \|M\|_{2} +ch^3\leq \beta  \| \e_{n}-\e_{n-1}\| +h c_1 \| \e_{n} \|   +ch^3.  
\end{align}
  Defining three quantities $d_{1}= \frac{c}{c_{1}} $, $d_{2} = \frac{2c_{1}}{1-\beta} $, and $d_{3} = \frac{2c}{1-\beta} $, now  we prove the following statement using principle of induction 
  \begin{equation}\label{eq:induction}
  \|\e_{n} \| \leq  d_{1} \e^{d_{2}hn} h^2,\quad \|\e_{n+1} - \e_{n}\| \leq d_{3}\e^{d_{2}hn}h^3,\ \quad n \geq 0.
  \end{equation}
  We first show that base case. When $n=0$, by definition we have $\|\e_0\|=\|\W\x(0)-\x^0\|=0$. And by \eqref{eq:error}, $\|\e_1-\e_0\|\leq ch^3<d_3h^3$, hence the induction base holds.
  
  Assume the proposition\eqref{eq:induction} holds for  $(n-1)$, that is, $\|\e_{n-1}\|\leq d_1\e^{d_2h(n-1)}h^2,\ \text{and}\ \|\e_{n} - \e_{n-1}\| \leq d_{3}\e^{d_{2}h(n-1)}h^3 $, then in the $n$th case,
  \begin{align*}
   \|\e_{n}\| &\leq   \|\e_{n-1}\| + \| \e_{n}-\e_{n-1}\|\\
   &  \leq d_{1} \e^{d_{2}h(n-1)} h^2 +  d_{3}\e^{d_{2}h(n-1)}h^3\\
   & = d_{1}(1+ \frac{d_{3}h}{d_{1}}) \e^{d_{2}h(n-1)} h^2 =  d_{1}(1+ d_{2}h) \e^{d_{2}h(n-1)} h^2 \\
   & \leq  d_{1} \e^{d_{2}hn} h^2
\end{align*}
and by \eqref{eq:iteration_gap},
\begin{align*}
    &  \| \e_{n+1}-\e_{n}\| \leq \beta d_{3}\e^{d_{2}h(n-1)}h^3+h c_{1} d_{1} \e^{d_{2}hn} h^2 +ch^3 \\
    & \leq d_{3}\underbrace{(\frac{d_{1}c_{1}}{d_{3}} + \beta +\frac{c}{d_{3}} )}_{=1} \e^{d_{2}hn} h^3 = d_{3}\e^{d_{2}hn}h^3.
\end{align*}
Then we have proven that the claim \eqref{eq:induction} also holds for the $n^{th}$ case. 

\end{proof}


\section{Proof of Theorem 4.1 (IGR-M)} 
\label{sec:2}
\begin{corollary}
\label{cor:GD+M}
 Let the loss $E$ for full-batch gradient descent be smooth and 4-times differentiable, then  GD-momentum update
\begin{align*}
& \x^{n+1} = \x^{n} - h \nabla E(\x^{n}) + \beta(\x^{n} -\x^{n-1}), & &  {n = 1,2,...,N}
\end{align*} 
is $O(h^2)$ close to
the flow of the piecewise ODE 
 \begin{align}
 \label{eq:gdm}
     \W\x'(t) = -\frac{1}{1-\beta} \nabla \hat{E}_n(\W\x(t)), \quad t\in[t_n, t_{n+1}], 
 \end{align}
where the modified loss is given as $$
\hat{E}_n(\W\x(t))=  (1-\beta^{n+1} )\nabla E(\x(t)) + \frac{h}{2(1-\beta)}\left[\frac{(1+\beta)}{(1-\beta)}(1-\beta^{2n+2})-4(n+1)\beta^{n+1}\right] \nabla^{2} E(\x(t)) \nabla E(\x(t)) .
$$
For later intervals when $k\gg \frac{\log h}{\log \beta} $ (we only ruled out a very small number of the initial intervals, as this lower bound grows very slowly (logarithmically) as $h\rightarrow 0$), the modified loss and the ODE both become independent of $k$, that is, the the later GD+momentum updates $\x^{n}$ is $O(h^2)$ close to $\W\x(t_n)$ which is the solution to 
\[
 \W\x'(t) = -\frac{1}{1-\beta} \nabla \hat{E}(\W\x(t)), 
\]
with
$\hat{E}(\W\x(t))=  E(\W\x(t)) + \frac{(1+\beta)h}{4(1-\beta)^2} \|\nabla  E(\W\x(t))\|^2$.
\end{corollary}

\begin{proof}
Corollary \eqref{eq:gdm} is a special case of Theorem \ref{thm:main_theorem}, and the result straightforwardly follows by setting $E_{n}=E $ for all $1\leq n\leq N$.

\end{proof}

\section{Proof of Remark 5.2}
\label{sec:3}
\begin{theorem}
The expectation of the IGRM for (SGD+M) taken over the draw of random batches is  $\ex(IGRM_{s}) = \frac{h(1+\beta)}{4(1-\beta)^3}\|\nabla E(\W\x(t))\|^2+\frac{h}{4(1-\beta)^2}\ex( \| \nabla E_{j}(\W\x(t)) - \nabla E(\W\x(t)) \|^2)$. This expected IGR for (SGD+M) is greater than that of the expected IGR for (SGD) taken with respect to the draw of the batches. 
\end{theorem}

\begin{proof}
To avoid confusion of notation use, we use $\ex$ as the symbol for expectation and $E$ as the loss-function.  Here the operator $\ex$ denotes the expectation is with respect to the draw of all the random batches. 
Recalling from \eqref{ode_mod}, the implicit regularizer while updating the $n^{th}$ mini-batch was of the form:
\begin{align}
    & IGRM_{s} = \frac{h}{4}\left(\|\nabla G_n(\W\x(t))\|^2+2\sum_{k=0}^{n-1}\beta^{n-k}\|\nabla G_k(\W\x(t))\|^2\right),\notag \\ 
    & \ex(IGRM_{s}) =  \frac{h}{4}\left(\ex(\|\nabla G_n(\W\x(t))\|^2) + 2\sum_{k=0}^{n-1} \beta^{n-k}\ex( \|\nabla G_k(\W\x(t))\|^2)   \right).  \label{use}
\end{align}
Here as the expectation $\mathbb{E}$ is taken over the draw of random batches,
we first derive $\ex(\|\nabla G_n(\W\x)\|^2)$ as follows
\begin{align}
    & \ex\left(\|\nabla G_n(\W\x(t))\|^2\right)\notag \\
    &\quad= \ex\left(\big(\sum_{k=0}^n \beta^{n-k} \nabla E_{k}(\W\x(t))\big)^T\big(\sum_{r=0}^n \beta^{n-r} \nabla  E_{k}(\W\x(t))\big) \right)\notag\\
    & \quad= \sum_{i=0}^{n} \beta^{2n-2i} \ex\left(\|\nabla E_{i}(\W\x(t))  \|^2\right) + \sum_{k=0}^n \sum_{r=0, k\neq r}^n \beta^{2n-r-k}\ex \left( \nabla E_{k}(\W\x(t))^T \nabla E_{r}(\W\x(t)) \right)\notag \\ 
    & \quad=  \sum_{i=0}^{n} \beta^{2n-2i} \ex(\|\nabla E_{i}(\W\x(t_{i}))  \|^2) + \underbrace{\sum_{k=0}^n \sum_{r=0, k\neq r}^n \beta^{2n-r-k}\ex \left( \nabla E_{k}(\W\x(t_{k}))^T \nabla E_{r}(\W\x(t_{r})) \right)}_{\RomanNumeralCaps{3}} +O(h).\notag  \label{eq:intme}\\ 
\end{align}
We obtain the last step by replacing the variables $\W\x(t)$ by $\W\x(t_{i})$, $\W\x(t_{k})$ and $\W\x(t_{r})$, respectively. We note that changing from  $\W\x(t)$ to  $\W\x(t_{i})$ introduces an $O(h)$ error from Taylor series. This $O(h)$ error gets multiplied with the coefficient $\frac{h}{4}$ in front of the regularizer. This overall $O(h^2)$ error does not affect the regularizer because it is of $O(h)$. 

Note that the random selection of $i^{th}$ mini-batch $E_i$ is independent of $\W\x(t_i)$, we calculate $\RomanNumeralCaps{3}$ as 
\[
\RomanNumeralCaps{3} = \sum_{k=0}^n\sum_{r=0,k\neq r}^n\beta^{2n-r-k}\nabla E(\W\x(t_k))^T\nabla E(\W\x(t_r)).
\]
Recall that the full-batch gradient loss is defined as $\nabla E(\W\x) = \frac{1}{M} \sum_{j=1}^M \nabla E_{(j)}(\W\x) $, where $E_{(j)}$ is the $j$th mini-batch as in (7) of the main paper. The summands in the first term in \eqref{eq:intme} become
\begin{align}
    \ex(\|\nabla E_{i}(\W\x(t_{i}))\|^2 )  =  \| \nabla E(\W\x(t_{i}))\|^2 + \ex \| \nabla E_{i}(\W\x(t_{i})) - \nabla E(\W\x(t_{i})) \|^2.
    \end{align}
We can calculate $\ex(\|\nabla G_n(\W\x(t))\|^2)$ as follows
\begin{align*}
     &\ex(\|\nabla G_n(\W\x(t))\|^2) \\
     &\quad=   \sum_{i=0}^{n} \beta^{2n-2i} \ex(\|\nabla E_{i}(\W\x(t_{i}))  \|^2) + \sum_{k=0}^n\sum_{r=0,k\neq r}^n\beta^{2n-r-k}\nabla E(\W\x(t_k))^T\nabla E(\W\x(t_r))+O(h)\\
     &\quad=  \sum_{i=0}^{n} \beta^{2n-2i} \left( \| \nabla E(\W\x(t_{i}))\|^2 + \ex( \| \nabla E_{i}(\W\x(t_{i})) - \nabla E(\W\x(t_{i})) \|^2)  \right)\\
     &\quad\quad+ \sum_{k=0}^n\sum_{r=0,k\neq r}^n\beta^{2n-r-k}\nabla E(\W\x(t_k))^T\nabla E(\W\x(t_r))+O(h)  \\
    & \quad = \sum_{k=0}^n\sum_{r=0}^n\beta^{2n-r-k}\nabla E(\W\x(t_k))^T\nabla E(\W\x(t_r)) + \sum_{i=0}^n\beta^{2n-2i}\ex(\|\nabla E_i(\W\x(t_i))-\nabla E(\W\x(t_i))\|^2)+O(h) \\ 
    & \quad= \left\|\sum_{k=0}^n\beta^{n-k}\nabla E(\W\x(t))\right\|^2+ \sum_{i=0}^n\beta^{2n-2i}\ex\left(\|\nabla E_i(\W\x(t))-\nabla E(\W\x(t))\|^2\right)+ O(h). \\ 
    &\quad =\left(\frac{1-\beta^{n+1}}{1-\beta}\right)^2\|\nabla E(\W\x(t))\|^2+\frac{1-\beta^{2n+2}}{1-\beta^2}\ex\left( \| \nabla E_{i}(\W\x(t)) - \nabla E(\W\x(t)) \|^2\right)+O(h).
\end{align*}
We write the second to last line similarly as before because changing $\W\x(t_{k})$ to $\W\x(t)$ only introduces $O(h)$ error. 

Similarly for any such $k$, we will have 
\[\ex(\|\nabla G_k(\W\x(t))\|^2) = \left(\frac{1-\beta^{k+1}}{1-\beta}\right)^2\|\nabla E(\W\x(t))\|^2+\frac{1-\beta^{2k+2}}{1-\beta^2}\ex( \| \nabla E_{i}(\W\x(t)) - \nabla E(\W\x(t)) \|^2)+O(h).
\]
Putting the expression for $\ex(\|\nabla G_n(\W\x(t))\|^2) $ and $\ex(\|\nabla G_k(\W\x(t))\|^2) $ into \ref{use}, we get:
\begin{align}
    & \ex(IGRM_{s})\\
    &\quad \approx  \frac{h}{4}\left(\ex(\|\nabla G_n(\W\x(t))\|^2) + 2\sum_{k=0}^{n-1}\beta^{n-k} \ex( \|\nabla G_k(\W\x(t))\|^2)   \right) \\ 
    & \quad= \frac{h}{4}\| \nabla E(\W\x(t))\|^2  \left(\left(\frac{1-\beta^{n+1}}{1-\beta}\right)^2 + 2\sum_{k=0}^{n-1} \beta^{n-k} \left(\frac{1-\beta^{k+1}}{1-\beta}\right)^2\right) \\ 
    &\quad\quad+\frac{h}{4}\ex( \| \nabla E_{i}(\W\x(t)) - \nabla E(\W\x(t)) \|^2)\left(\frac{1-\beta^{2n+2}}{1-\beta^2}+2\sum_{k=0}^{n-1}\beta^{n-k}\frac{1-\beta^{2k+2}}{1-\beta^2}\right)+O(h^2)\label{eq:IGRM_SGD}
\end{align}
For large number of iterations $n$,  \eqref{eq:IGRM_SGD} reduces to 
\[
\mathbb{E}(IGRM_{s})=\frac{h(1+\beta)}{4(1-\beta)^3}\|\nabla E(\W\x(t))\|^2+\frac{h}{4(1-\beta)^2}\ex( \| \nabla E_{i}(\W\x(t)) - \nabla E(\W\x(t)) \|^2).
\]
\end{proof}
\section{Proof of Remark 5.3}
\label{sec:4}
\begin{theorem}
Let $\C$ be the covariance matrix of the driving force of (SGD) at the $k^{th}$ iteration ,i.e, $cov(\nabla E_{k}(\x))= \C \in \mathbb{R}^{p \times p} $ then the covariance matrix for the driving force for (SGD+M) (with adjusted learning rate) is $cov((1-\beta)\nabla G_{k}(\x)) = \frac{1-\beta}{1+\beta} \C$. Here the random vectors $E_{k}$ and $G_{k}$ are evaluated at a fixed point $\x$.  
\end{theorem}

\begin{proof}
Assuming the stochastic gradient $\nabla E_{k}(\x)$ is sampled from a  distribution with i.i.d entries where the mean is the full-batch gradient $\nabla E(\x)$ and the covariance matrix is $\C$. Then by definition,
\begin{align*}
    & \ex((1-\beta)\nabla G_{k}(\x)) = (1-\beta)\sum_{i=0}^{k}\beta^{k-i} \ex(\nabla E_{i}(\x)) = (1-\beta^{k+1}) \nabla E(\x)  \approx \nabla E(\x).
\end{align*}
From definition of $\C$ we have 
\begin{align}
   &  \C =   \ex((\nabla E_{k}(\x)-\nabla E(\x))(\nabla E_{k}(\x)-\nabla E(\x))^T) 
    = \ex(\nabla E_{k}(\x) \nabla E_{k}(\x)^T )- \nabla E(\x)\nabla E(\x)^T. \label{int7}
\end{align}
Then the covariance matrix for (SGD+M) is :
\begin{align}
    & cov((1-\beta) G_{k}(\x))\notag \\
    &\quad= \ex(((1-\beta)\nabla G_{k}(\x)-\ex((1-\beta)\nabla G_{k}(\x))) ((1-\beta)\nabla G_{k}(\x)-\ex((1-\beta)\nabla G_{k}(\x))^T   ) \notag \\
    & \quad= \ex(\nabla G_{k}(\x) \nabla G_{k}(\x)^T)(1-\beta)^2 - \ex((1-\beta)\nabla G_{k}(\x)) \ex((1-\beta)\nabla G_{k}(\x))^T \notag \\ 
    & \quad \approx  \ex(\nabla G_{k}(\x) \nabla G_{k}(\x)^T)(1-\beta)^2 - \nabla E(\x)\nabla E(\x)^T.  \label{int10}
\end{align}
Now let's evaluate $ \ex(\nabla G_{k}(\x) \nabla G_{k}(\x)^T)$ as follows:
\begin{align}
    & \ex(\nabla G_{k}(\x) \nabla G_{k}(\x)^T)\notag \\
    &\quad= \ex((\sum_{k=0}^n \beta^{n-k}\nabla E_{k}(\x)) (\sum_{k=0}^n \beta^{n-k}\nabla E_{k}(\x))^T  ) \notag \\
    & \quad= \sum_{p=0}^n \beta^{2n-2p} \underbrace{\ex(\nabla E_{p}(\x) \nabla E_{p}(\x)^T)}_{\C+\nabla E(\x)\nabla E(\x)^T \text{from \ref{int7}}} + \sum_{i=0}^n \sum_{j=0,j\neq i}^n \beta^{2n-i-j} \underbrace{\ex( \nabla E_{i}(\x)\nabla E_{j}(\x)^T)}_{\nabla E(\x)\nabla E(\x)^T } \notag \\
    & \quad= (\C+\nabla E(\x)\nabla E(\x)^T )\frac{(1-\beta^{2k+2})}{(1-\beta^2)} + \nabla E(\x)\nabla E(\x)^T [(\frac{1-\beta^{k+1}}{1-\beta})^2 - \frac{1-\beta^{2k+2}}{1-\beta^2}] \notag \\
    & \quad= \frac{1-\beta^{2k+2}}{1-\beta^2} \C + (\frac{1-\beta^{k+1}}{1-\beta})^2 \nabla E(\x)\nabla E(\x)^T.  \label{int8}
\end{align}

Putting \ref{int8} into \ref{int10}, we have:
\begin{align*}
     & cov((1-\beta) G_{k}(\x)) \\
     &\quad \approx \ex(\nabla G_{k}(\x) \nabla G_{k}(\x)^T)(1-\beta)^2 - \nabla E(\x)\nabla E(\x)^T  \notag\\ 
     & \quad\textbf{}= \left( \frac{1-\beta^{2k+2}}{1-\beta^2} \C + (\frac{1-\beta^{k+1}}{1-\beta})^2  \nabla E(\x)\nabla E(\x)^T\right) (1-\beta)^2 -\nabla E(\x)\nabla E(\x)^T \notag \\ 
     & \quad\textbf{}= \frac{(1-\beta^{2k+2})(1-\beta)^2}{1-\beta^2} \C - \beta^{k+1}(2-\beta^{k+1})\nabla E(\x)\nabla E(\x)^T. \notag
\end{align*}
For a high enough iteration $k$, it reduces to:

\begin{align*}
    cov((1-\beta) G_{k}(\x)) = \frac{(1-\beta)^2}{1-\beta^2} \C = \frac{1-\beta}{1+\beta} \C.
\end{align*}

\end{proof}

\newpage

\section{Additional Experiments}
We delay the result for CIFAR-100 classification here in the appendix. The final test accuracy is reported in Table-1 in the manuscript. 

\subsection{CIFAR-100 classification results}
\begin{figure*}[h]

        \centering
        \begin{subfigure}[b]{0.375\textwidth}
            \centering
            \includegraphics[width=\textwidth]{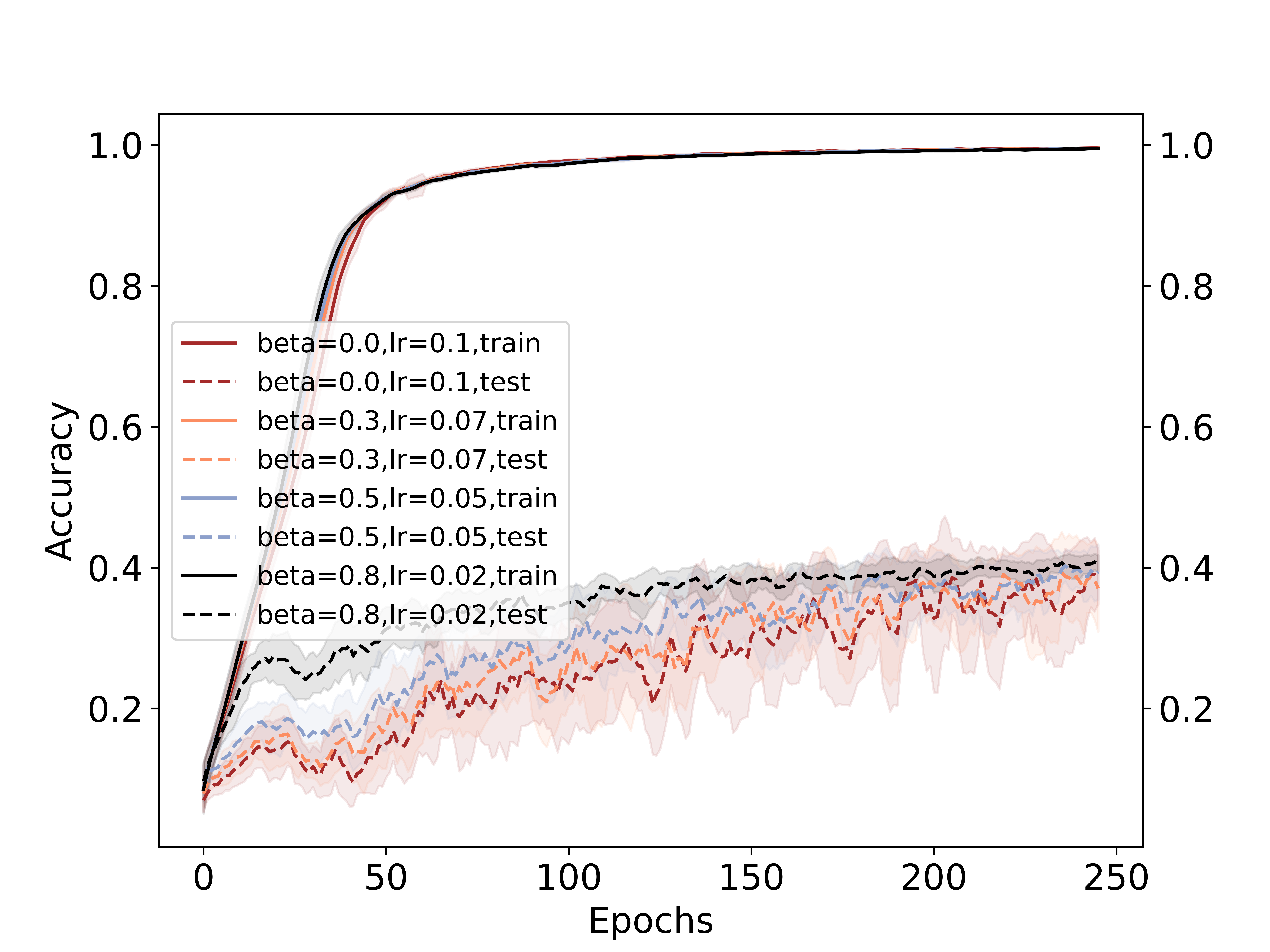}
            \caption[Wideresnet]%
            {{\small WideresNet-16-8}}    
            \label{fig:mean and std of net14}
        \end{subfigure}
        \begin{subfigure}[b]{0.375\textwidth}  
            \centering 
            \includegraphics[width=\textwidth]{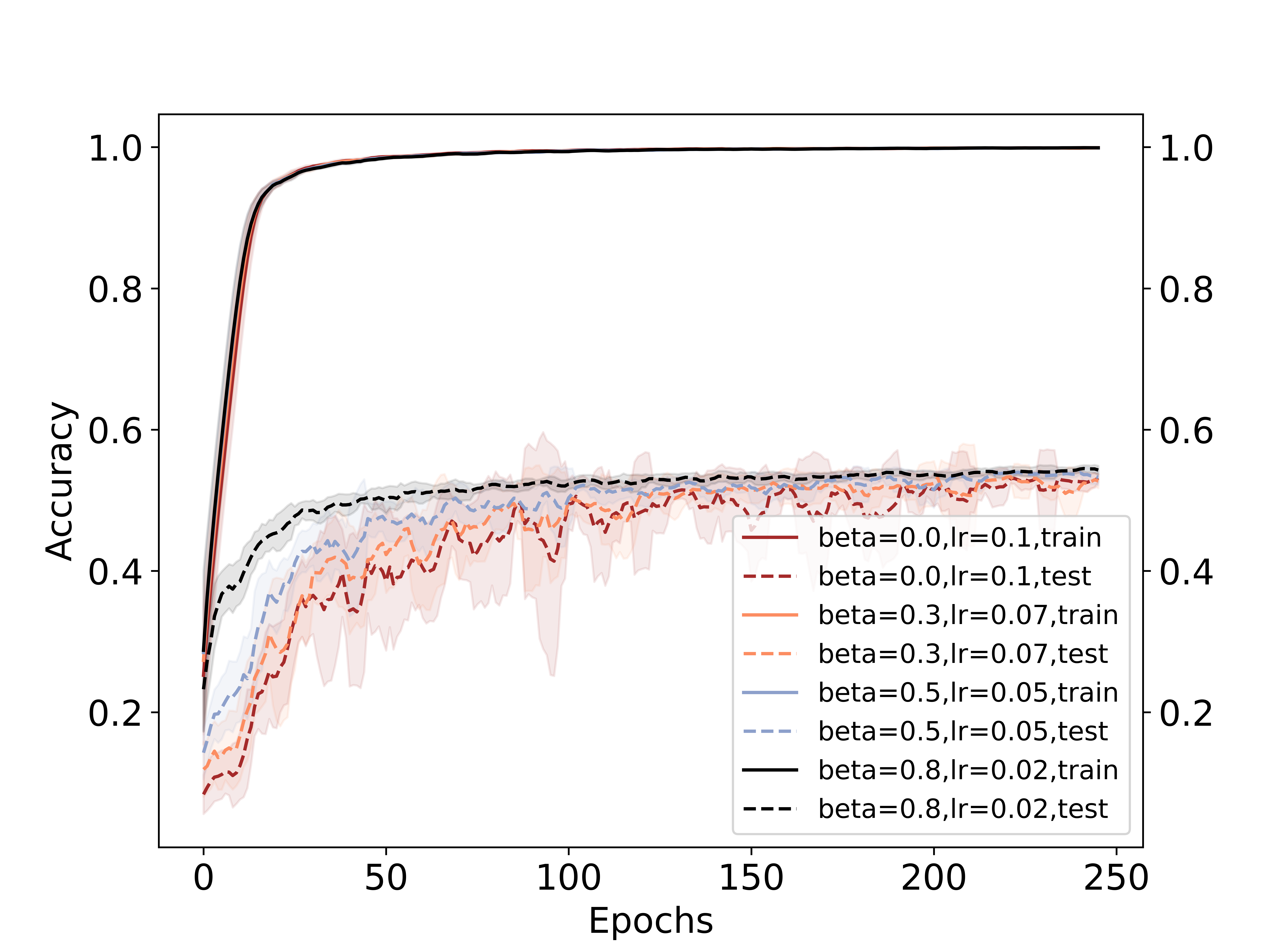}
            \caption[ Resnet-18]%
            {{\small ResNet-18}}    
            \label{fig:mean and std of net24}
        \end{subfigure}
        \vskip\baselineskip \vspace{-0.15 in}
        \begin{subfigure}[b]{0.375\textwidth}   
            \centering 
            \includegraphics[width=\textwidth]{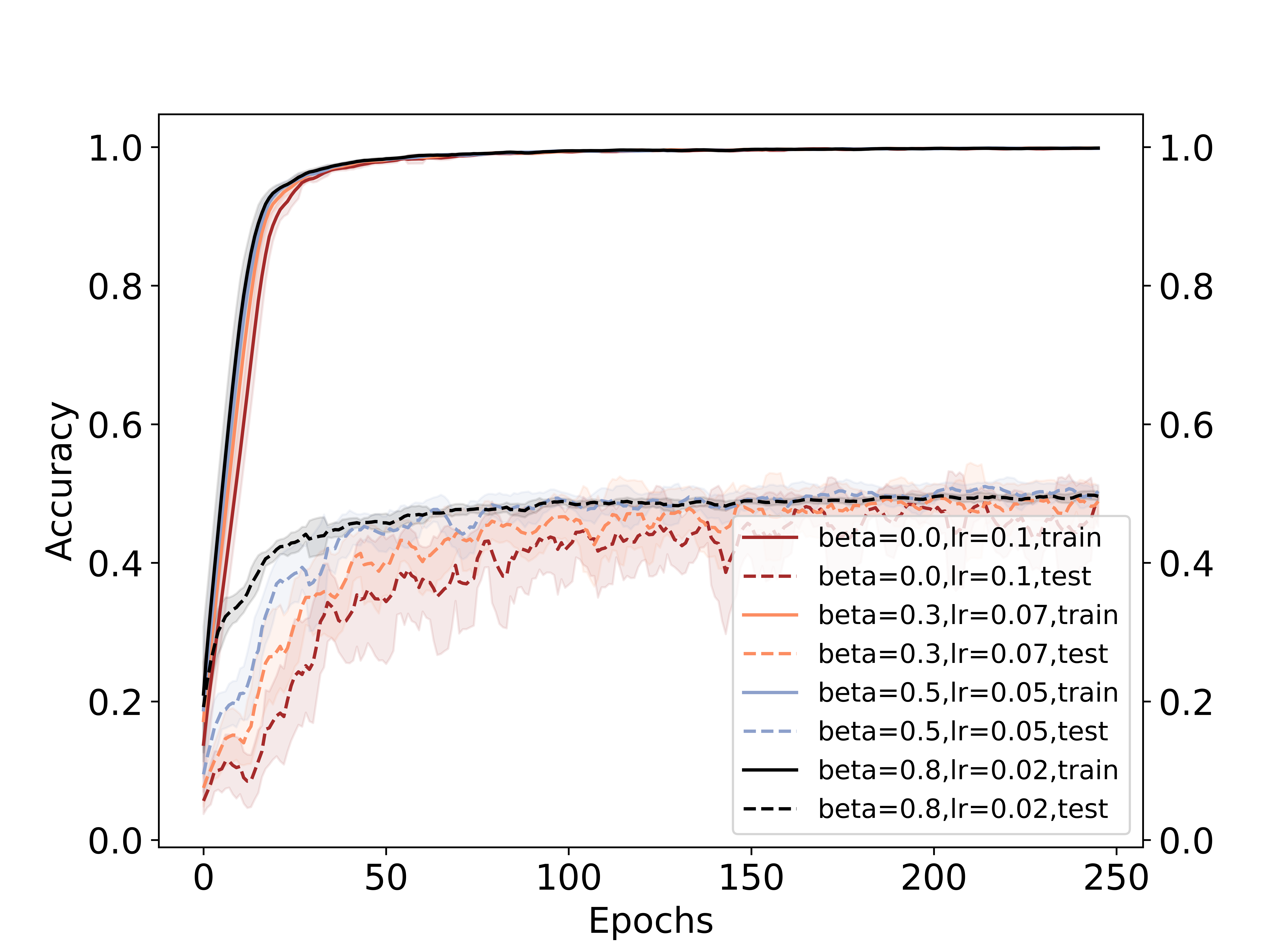}
            \caption[Resnet-50]%
            {{\small ResNet-50}}    
            \label{fig:mean and std of net34}
        \end{subfigure}
        \begin{subfigure}[b]{0.375\textwidth}   
            \centering 
            \includegraphics[width=\textwidth]{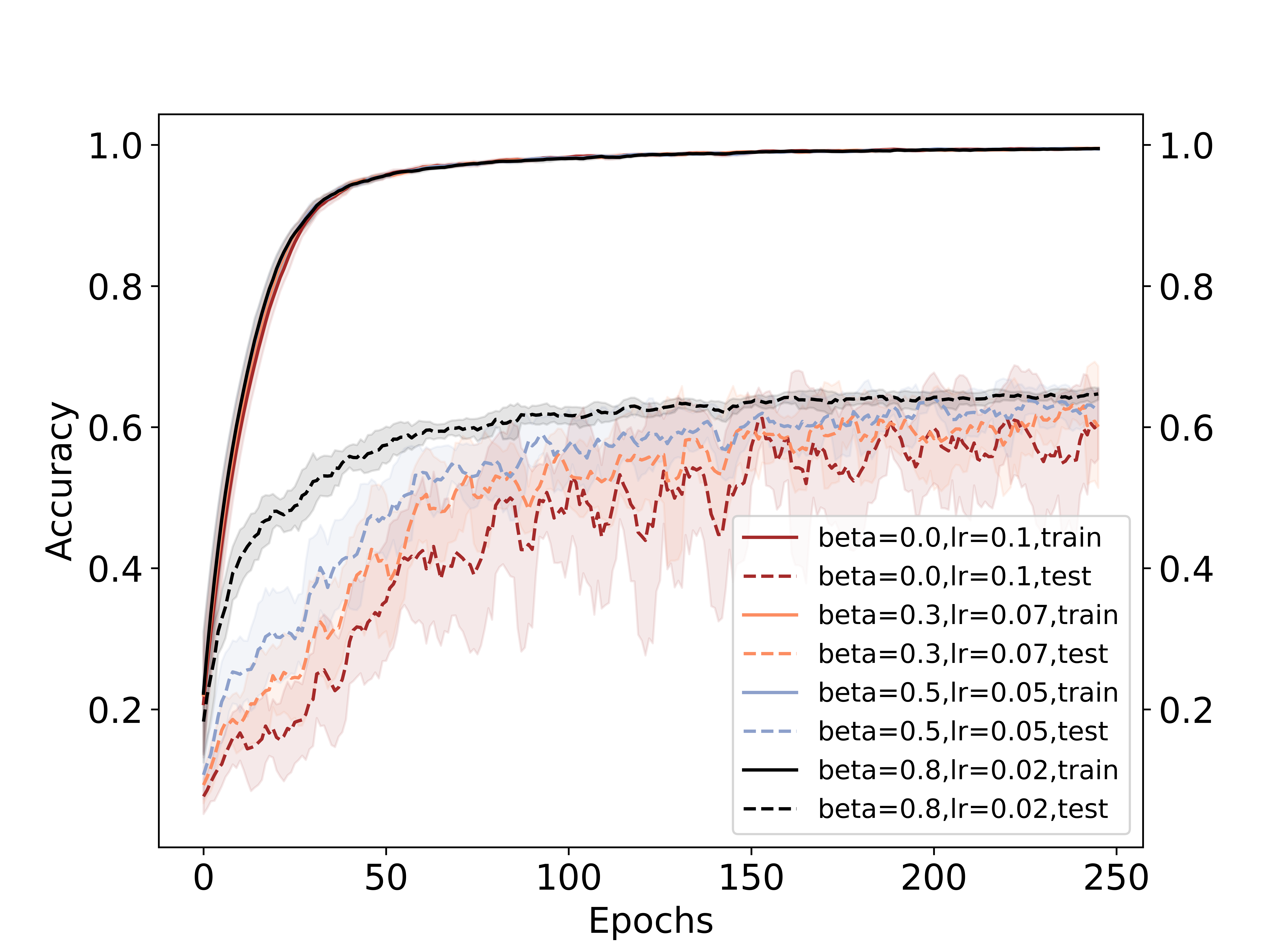}
            \caption[Densenet-121]%
            {{\small DenseNet-121}}    
            \label{fig:mean and std of net44}
        \end{subfigure}
        \caption[ The average and standard deviation of critical parameters ]
         {\small Classification results for CIFAR-100 dataset with various network architectures with combinations of $(h,\beta)$ chosen such that the effective learning rate $\frac{h}{(1-\beta)}$ remains same. In all of the experiments, external regularization like weight-decay, l.r scheduler, dropout,label-smoothing are kept off (except Batch-normalization). The results have been averaged over 3 random seeds having different initializations}.  
        \label{SGD+Mfig}
    \end{figure*}

\subsection{Effect of learning-rate scheduler}

Learning rate schedulers are a common practice in training classification networks hence exploring the effect of IGR and IGR-M in schedulers is important. In the experiment, we train a Resnet-18 and a Resnet-50 network to classify CIFAR-10 dataset to compare the performance of (SGD) and (SGD+M) under the effect of learning rate scheduler. 
     \begin{figure*}[t]
        \centering
        \begin{subfigure}[b]{0.395\textwidth}
            \centering
            \includegraphics[width=\textwidth]{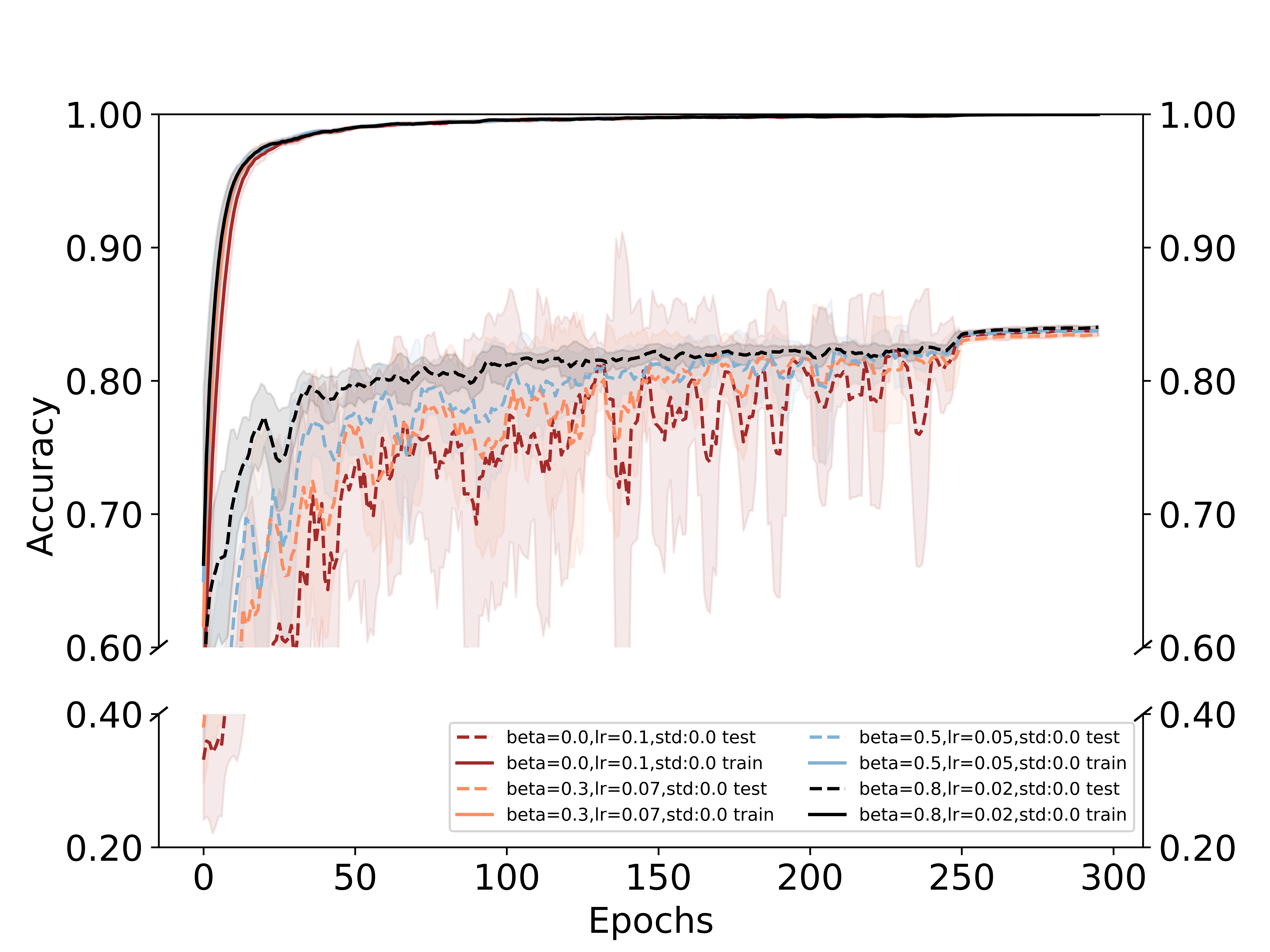}
            \caption[Wideresnet]%
            {{\small Resnet-18 }}    
            \label{lrschd2}
        \end{subfigure}
        \begin{subfigure}[b]{0.395\textwidth}  
            \centering 
            \includegraphics[width=\textwidth]{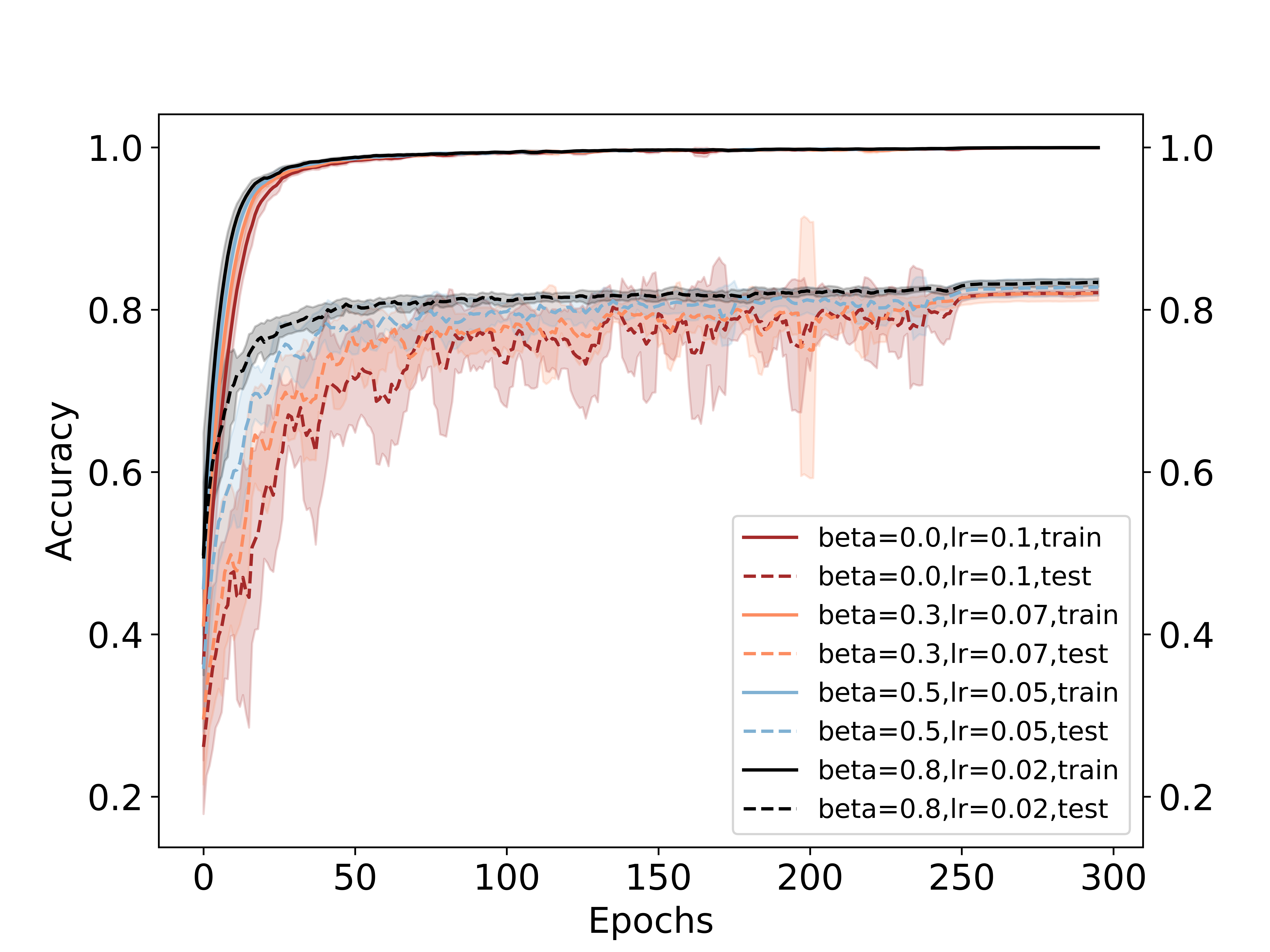}
            \caption[ Resnet-18]%
            {{\small Resnet-50}}    
            \label{fig:mean and std of net24 2}
        \end{subfigure}
        \caption[ The average and standard deviation of critical parameters ]
        {\small Classification result of CIFAR-10 with step scheduler ($\frac{1}{10}$) activated at epoch =250 with various $\beta$ but the same effective learning rate $\frac{h}{(1-\beta)}$.   } 
        \label{fig:lrschd1}
\end{figure*}

We observe that just like our previous experiments comparing (SGD) and (SGD+M), the test accuracy is higher with increasing $\beta$. We attributed this effect due to the stronger implicit regularization for momentum than plain SGD. However, after the effect of scheduler, the learning rate is decreased by a factor of 10. This diminishes the effect of the implicit regularizer for both SGD and SGD+M as $IGR \propto h$. 
However, from empirical observations (Fig-2) the difference in test accuracy of (SGD) and (SGD+M) (near convergence) still exists but may not be in a pronounced way as the initial iterations. We believe this is because during the earlier iterations, the significantly stronger IGR for (SGD+M) guides it's trajectory through flatter sub-manifolds than that of (SGD). The effect is prominent enough that even after scheduler is activated (also near convergence), (SGD+M) still has a slightly higher test accuracy than (SGD).\\

\section{In the stable regime, convergence-rate of (GD+M) is $\frac{1}{(1-\beta)}$ larger than (GD)  using classical convergence analysis}
In this section, we show that when (GD) with a learning-rate $h$ and (GD+M) with an effective learning rate $\frac{h}{(1-\beta)}$ both fall inside the stable regime of (GD),then the convergence-rate of (GD+M) is $\frac{1}{(1-\beta)}$ larger than (GD).

Classical convergence of (GD) and (GD+M) is considered in a locally quadratic surface. On a standard quadratic, the minimization is $\min_{\x} f(\x)= \frac{1}{2} \x^T \A \x -b^T \x +c$, where $\A$ is positive semi-definite matrix with eigen-values in $[\mu,L]$. A simple change of variable would mean doing a minimization of the form $\min_{\x} \frac{1}{2} \x^T \Sigma \x $, where $\Sigma$ contains the eigenvalues of A on the diagonal. Hence $\nabla f(\x) = \Sigma \x$ and $\nabla^2 f(\x) = \Sigma$. Furthermore, the condition number of the objective function is denoted as $\kappa = \frac{L}{\mu}$. 

For Heavy-Ball method, the iterates follow:

\begin{align}
    \x^{k+1} =\x^{k}-h\nabla f(\x^{k}) +\beta (\x^{k}- \x^{k-1})
\end{align}
On a locally quadratic, the iterates roughly follow
\begin{align}
     \x^{k+1} =\x^{k}-h\Sigma \x  +\beta (\x^{k}- \x^{k-1}) = ((1+\beta)\I - h \Sigma)\x^{k} -\beta \x^{k-1}
\end{align}
With slight rearrangement, which could be written as :

\begin{align}
   \begin{bmatrix}
\x^{k+1}\\
\x^{k}
\end{bmatrix} = \begin{bmatrix}
(1+\beta)\I -h\Sigma & -\beta \I \\
\I & \mathbf{0} 
\end{bmatrix} \begin{bmatrix}
\x^{k}\\
\x^{k-1}
\end{bmatrix}
\end{align}

Denoting $\y^{k}=  \begin{bmatrix}
\x^{k+1}\\
\x^{k}
\end{bmatrix}$ and $\To= \begin{bmatrix}
(1+\beta)\I -h\Sigma & -\beta \I \\
\I & \mathbf{0} 
\end{bmatrix} $, the norm of $\|\y^{k} \|_{2} $ is derived as follows:

\begin{align}
    \| \y^{k}\| =  \|\To \y^{k-1}\|=   \|\To^{k} \y^{0}\| \leq  \|\To^{k}\|_{2} \| \y^{0}\| \leq (\rho(\To))^k \kappa(V) \| \y^{0} \|
\end{align}
where $\rho(\To) $ is the spectral radius of $\To$ and $\To$ has an eigen-decomposition $\To=VDV^{-1}$, $\kappa(V)$ being the condition number of $V$. $\To$ is permutation-similar to the block-diagonal matrix $\To = \begin{bmatrix}
\To_{1} & \mathbf{0} &.& . & \mathbf{0}\\
 \mathbf{0} & \To_{2} & .& . & \mathbf{0}\\
. & . & . & .&  .\\
 \mathbf{0} & \mathbf{0} &. &. & \To_{n}\\
\end{bmatrix} $, where $\To_{j} = \begin{bmatrix}
1+\beta-\alpha \lambda_{j} & -\beta\\
1 & 0\\
\end{bmatrix} $ is a $2\times2 $ matrix for $j=1,2..n $. Letting $r_{j}$ denote the  eigen-values for each block matrix $\To_{j} $ and would satisfy \\$    r_{j}= 
\begin{cases}
   \frac{1}{2}((1+\beta -\alpha \lambda_{j})\pm \sqrt{ (1+\beta-h\lambda_{j})^2 -4\beta}),& \text{if } (1+\beta-h\lambda_{j})^2 -4\beta=\Delta_{j} > 0 \\
     \frac{1}{2}((1+\beta -\alpha \lambda_{j}) \pm i \sqrt{|\Delta_{j}|},              & \text{otherwise}
\end{cases} $
where $i=\sqrt{-1}$.
Due to the block-matrix structure of $\To$, the convergence factor $\rho(\To)$ is determined by the largest vectors among all the block matrices $\To_{j}$, i.e, $\rho(\To) =\max _{j} r_{j}= \max r_{1},r_{n}$. 

Now depending upon the 4 conditions $\Delta_{j}\leq 0 \equiv \beta \geq (1-\sqrt{h \lambda_{j}}) $,  $\Delta_{j}> 0 \equiv \beta \leq  (1-\sqrt{h \lambda_{j}}) $ , $|1-\sqrt{h\mu}| < |1-\sqrt{h L}|$ and $|1-\sqrt{h\mu}| >|1-\sqrt{h L}|$, we have four sub-cases to determine $\rho(\To)$:
\begin{enumerate}
    \item If $0< h \leq (\frac{2}{\sqrt{L}+\sqrt{\mu}})^2 $ and $\beta \geq  (1-\sqrt{h\mu})^2$ 
    \item  If $0< h \leq (\frac{2}{\sqrt{L}+\sqrt{\mu}})^2 $ and $\beta <  (1-\sqrt{h\mu})^2$ 
    \item $h >(\frac{2}{\sqrt{L}+\sqrt{\mu}})^2 $ and $\beta \geq (\sqrt{h L}-1)^2 $
    \item  $h >(\frac{2}{\sqrt{L}+\sqrt{\mu}})^2 $ and $\beta < (\sqrt{h L}-1)^2 $
\end{enumerate}

For a small $h$ and fixed $\beta$, satisfies condition-2 and the effective learning rate lies in the stability regime of GD. 
Under this particular condition (2), we have $\Delta_{1}>0 $, hence the spectral radius  $\rho(\To)$ becomes (by taking the larger $r_{j}$) :

\begin{align}
   & \rho^{(GD+M)} = \frac{1}{2}(1+\beta -h \mu + \sqrt{(1+\beta-h\mu)^2 -4\beta}) \quad  \text{[considering the larger term]}\\
   & = \frac{1}{2}(1+\beta -h \mu + \sqrt{(1-\beta)^2-2h\mu(1+\beta) +h^2\mu^2 }) \\
   &  = \frac{1}{2}(1+\beta -h \mu + (1-\beta)(\underbrace{\sqrt{1-\frac{2h\mu(1+\beta) +h^2\mu^2}{(1-\beta)^2} }}_{1-\frac{1}{2}\frac{2h\mu(1+\beta)}{(1-\beta)^2} +O(h^2) }  -1 ) + (1-\beta))\\
   & \approx \frac{1}{2} (1+\beta -h \mu -\frac{h\mu(1+\beta)}{(1-\beta)} +  (1-\beta) ) \quad  \text{[small $h$ approximation]}\\ 
   & = 1- \frac{h \mu}{(1-\beta)}
\end{align}

Similarly, for (GD) with learning-rate $\Tilde{h}$ minimizing a locally quadratic function, using the classical convergence approach, we have $\| \x^{k} \| \leq \rho^{k}_{\Tilde{h}} \| \x^{0} \| $ where $ \rho_{\Tilde{h}} =\max (|1-\Tilde{h}\mu|,|1-\Tilde{h}L| )$. 
Hence for a small enough $h$ i.e,( $0< \Tilde{h} \leq \frac{2}{L+\mu} $), we have for the convergence rate for GD to be :
\begin{align}
    & \rho^{GD} = 1-\Tilde{h}\mu 
\end{align}

Putting $\Tilde{h} = \frac{h}{(1-\beta)}$, we see that $\rho^{(GD+M)} \approx \rho^{(GD)} $. Which means if we use a learning rate $\frac{1}{(1-\beta)}$ times larger for GD, it will match the convergence rate of (GD+M). 

Equivalently under the same learning rate for (GD) and (GD+M) (say $h$), the convergence rate of (GD+M) is $\frac{1}{(1-\beta)}$ times larger than that of (GD),i.e, $\rho^{(GD+M)} \approx \frac{1}{(1-\beta)} \rho^{(GD)} $.

\section{Role of variance in mini-batch gradients in finding better minima}
Losses of deep neural network are usually highly non-convex containing a lot of local minima. A good optimizer should have the ability of escaping local and bad (i.e., sharp) minimizers to settle for a good/flat minimum.  In SGD, the mini-batch gradient can be thought of as a noisy version of the full-batch gradient: $\nabla E_{i}(x) = \nabla E(x)  +\eta_i$. So, when an optimizer is stuck in a valley having a bad local minima, the randomness in the noisy gradient $\nabla E_{i}(x)$ provides a possibility of  \textbf{escaping} the valley (having a bad local minima). Very recently, this intuition has been mathematically formalized by (Ibayashi and Imaizumi (2021)). In their Theorem 2, the authors showed that the escape efficiency (reciprocal of mean exit time) of SGD  is $ \propto \exp({-\frac{B}{h}\Delta E \lambda_{max}^{-\frac{1}{2}}}) $,  where $B$, $h$, $\Delta E$ and $\lambda_{max}$ denote batch-size, learning rate, depth of minima and the largest eigenvalue of the Hessian,  respectively. In short, a smaller batch-size (B) and a larger learning rate are crucial to escaping bad local minima.

\section{IGR-M in 2D model with non-linear (sigmoid) activation}
Beyond the linear case in Section-4.1 of the manuscript, now we consider a 2D nonlinear model that has a Sigmoid activation function to explore the effect of IGR-M. The loss function $E$ is minimized using two learnable parameters $(w_{1},w_{2})$ but with a sigmoid layer in-between. Here the optimization problem is as follows:
\begin{align*}
    (\hat{w_{1}},\hat{w_{2}}) = \argmin_{w_{1},w_{2}} \frac{1}{2} (y-w_1 \sigma(w_{2}x))^2 \equiv \argmin_{w_{1},w_{2}} \frac{1}{2} \left(y-\frac{w_{1}}{1+e^{-w_{2}x}}\right)^2 := E(w_{1},w_{2})
\end{align*}
where $\sigma$ is the Sigmoid activation function.
The norm of the gradient has the following expression in this case:
\begin{align*}
    \|\nabla E \|^2 = \left|\frac{\partial E}{\partial w_{1}} \right|^2 +  \left|\frac{\partial E}{\partial w_{2}} \right|^2
    = \left(\frac{1}{(1+e^{-w_{2}x})^2} + \frac{w_{1}^2x^2e^{-2w_{2}x}}{(1+e^{-w_{2}x})^4}\right)\left(y-\frac{w_{1}}{1+e^{-w_2x}}\right)^2.
\end{align*}

\begin{wrapfigure}{r}{0.6\textwidth}
  \begin{center}
    \includegraphics[width=0.6\textwidth]{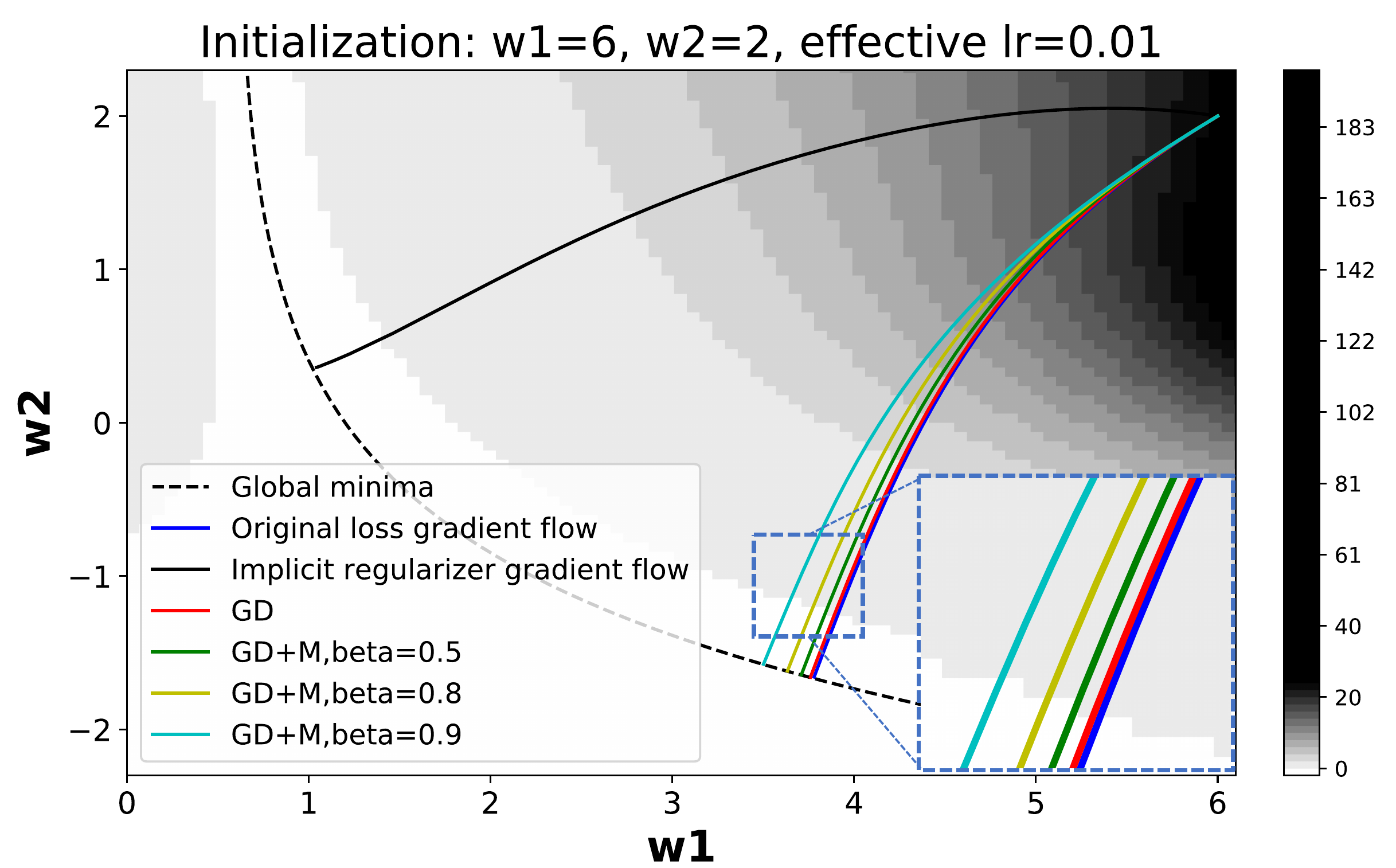}
  \end{center}
  \caption{Trajectories for (GD) and (GD+M) for various $\beta$ but with the same effective learning rate $\frac{h}{(1-\beta)}$. With increasing $\beta$, the trajectory becomes closer to the gradient flow of the implicit regularizer (solid back line), hence supporting our theory. The background color denotes the magnitude $\|\nabla E \|_{2}^2$ }
  \label{fig:IGR-nonlinear}
\end{wrapfigure}
The dashed black curve plots global minima given by the equation $w_{2} = -\frac{\log(\frac{w_{1}}{y}-1)}{x}$. Unlike the linear case, (where the IGR was proportional to the norm of the weights $w_{1}$ and $w_{2}$), here the IGR $\| \nabla E\|^2$ has a more complicated level set (Figure \ref{fig:IGR-nonlinear}. So, to help understand the effect of IGR-M, we plot two reference curves, one is the dark blue curve that represents the gradient flow for the original loss function,
given as 
\begin{align*}
    \x'(t) = -\nabla E(\x(t)).
\end{align*}
where $\x = [w_{1},w_{2}]^T$.
The other is the solid black curve that shows the gradient flow for implicit regularizer $\|\nabla E\|^2$ given as:
\begin{align*}
    \x'(t) = -\nabla \| \nabla E(\x(t))\|_{2}^2.
\end{align*}

A method with a stronger IGR would have a trajectory closer to the solid black curve.
So, we plot the trajectories of (GD) ($\beta=0$) and (GD+M) with $\beta=0.5,0.8$, and $0.9$, with the same initialization $(w_1=6, w_2=2)$. The effective learning-rate is kept the same for all the trajectories which equals the learning rate for GD, i.e., $\frac{h}{1-\beta}= 0.01$. We see how the trajectory for (GD+M) is closer to the gradient flow for implicit regularizer (the solid black curve) than that of (GD).
More explicitly, we observe that all the trajectories converge to the curve of global minima. However, with a larger $\beta$, the trajectory becomes closer to the gradient flow minimizing $\|\nabla E\|^2$ (the solid black curve). This observation agrees with our theorem which states the modified loss is a weighted combination of the original loss $E$ and implicit regularizer $\|\nabla E\|^2$, and larger $\beta$ leads to a larger weight for the regularizer $\|\nabla E\|^2$, hence making it closer to the solid black curve.

\section{Future directions}

IGR, although a great tool for examining generalization, relies on low-order Taylor approximations that works well under small learning rates. In addition, the current analysis is based on fixed values of $\beta$ while letting $h\rightarrow 0$. In practice, to better guide the choice of hyper-parameters, a bound that is asymptotic in both $h$ and $\beta$  ($ h\rightarrow 0, \beta\rightarrow 1$) might be more helpful. We leave it as future work. 

\bibliography{iclrappendix}
\bibliographystyle{iclrappendix}